\documentclass{article}

\usepackage[preprint]{neurips_2024}

\PassOptionsToPackage{}{natbib}
\usepackage{mlmath}
\usepackage{lipsum}
\usepackage{color}
\usepackage{times}
\usepackage{graphicx}
\usepackage{subfigure}
\usepackage{enumerate}
\usepackage{amsmath,amsfonts,amssymb,mathtools,amsthm}
\usepackage{makecell}

\usepackage{algorithm}
\usepackage{algorithmic}

\usepackage[utf8]{inputenc} 
\usepackage[T1]{fontenc}    
\usepackage{hyperref}       
\usepackage{url}            
\usepackage{booktabs}       
\usepackage{amsfonts}       
\usepackage{nicefrac}       
\usepackage{microtype}      
\usepackage{xcolor}         

\newtheorem*{prop*}{Proposition}
\newtheorem{proposition}{Proposition}

\newtheorem{theorem}{Theorem}
\newtheorem*{theorem*}{Theorem}
\newtheorem{lemma}{Lemma}
\newtheorem*{lemma*}{Lemma}

\newtheorem*{property*}{Property}
\newtheorem*{remark}{Remark}
\newtheorem{definition}{Definition}
\newtheorem*{definition*}{Definition}
\newtheorem{corollary}{Corollary}
\newtheorem*{corollary*}{Corollary}

\newtheorem*{assumption*}{Assumption}
\newtheorem{example}{Example}

\title{Disentangle Sample Size and Initialization Effect on Perfect Generalization for Single-Neuron Target}

\author{
Jiajie Zhao\textsuperscript{\rm 1}, Zhiwei Bai\textsuperscript{\rm 1}, 
Yaoyu Zhang\textsuperscript{\rm 1,2}\thanks{Corresponding author: zhyy.sjtu@sjtu.edu.cn.}\\
\textsuperscript{\rm 1}  School of Mathematical Sciences, Institute of Natural Sciences, MOE-LSC, \\ Shanghai Jiao Tong University, Shanghai 200240, P.R. China. \\
\textsuperscript{\rm 2} Shanghai Center for Brain Science and Brain-Inspired Technology, Shanghai 200240, P.R. China\\
\{zjj0216, bai299, zhyy.sjtu\}@sjtu.edu.cn.
}

\begin{document}

\maketitle
\begin{abstract}
Overparameterized models like deep neural networks have the intriguing ability to recover target functions with fewer sampled data points than parameters~\citep{zhang2023optimistic}. To gain insights into this phenomenon, we concentrate on a single-neuron target recovery scenario, offering a systematic examination of how initialization and sample size influence the performance of two-layer neural networks.
Our experiments reveal that a smaller initialization scale is associated with improved generalization, and we identify a critical quantity called the "initial imbalance ratio" that governs training dynamics and generalization under small initialization, supported by theoretical proofs. Additionally, we empirically delineate two critical thresholds in sample size—termed the "optimistic sample size" and the "separation sample size"—that align with the theoretical frameworks established by~\cite{zhang2023optimistic,zhang2023structure}. Our results indicate a transition in the model's ability to recover the target function: below the optimistic sample size, recovery is unattainable; at the optimistic sample size, recovery becomes attainable albeit with a set of initialization of zero measure. Upon reaching the separation sample size, the set of initialization that can successfully recover the target function shifts from zero to positive measure. These insights, derived from a simplified context, provide a perspective on the intricate yet decipherable complexities of perfect generalization in overparameterized neural networks.
\end{abstract}

\section{Introduction}
In machine learning, a fundamental problem is to learn a function from data sampled from a target function $f^*$ with the goal of minimizing the generalization error. Traditional learning theory suggests that overparameterized models, where the number of parameters exceeds the number of sample points, are prone to overfitting and poor generalization ~\citep{vapnik1998adaptive,bartlett2002rademacher}. However, in practice, overparameterized deep neural networks often exhibit good generalization performance~\citep{breiman2018reflections,zhang2021understanding}. To demystify this generalization phenomenon, researchers have sought to devise theoretical complexity measures to determine an upper bound on the generalization gap. Many proposed complexity measures are predicated on worst-case analyses, assessing the most unfavorable generalization scenarios within a given hypothesis space. Nonetheless, empirical investigations often show a weak or nonexistent relationship between these theoretical predictors and the observed generalization performance of actual models~\citep{jiang2019fantastic}.

Recent research has explored a novel concept contrary to the worst-case scenario, referred to as the "optimistic estimate". This investigates the minimum number of samples that models need to exactly reconstruct the target function in the recoverable setting~\citep{zhang2022linear, zhang2023optimistic}. Their experimental findings indicate that with appropriate hyperparameter tuning, the number of data points required to recover the target function can approach, or even match, the proposed "optimistic sample size". Furthermore,~\citet{zhang2023structure} characterized the structure of the loss landscape of two-layer neural networks near global minima. They discovered that as the sample size reaches a certain threshold, called the "separation sample size", the set of parameters with zero generalization error $Q^*$, referred to as the target set, separates out (see Definition~\ref{def:separation} for the formal definition of separation). However, the target set $Q^*$ generally consists of different branches, and it remains unclear to which branch the actual training process will converge for different sample sizes and initialization.

In our study, we conduct a systematic exploration of the impact that initialization and sample size have on the dynamics and convergence results of a model. The challenge in studying neural network dynamics stems from its dependency on multiple factors, including the specific architecture, dataset, optimization technique, and initialization method. To dissect the global dynamics and generalization capabilities within overparameterized neural networks, we concentrate on a simplified scenario: the recovery of a single-neuron target. In our context, term "recovery" and "perfect generalization" are both identical to zero generalization error. Despite its simplicity, this scenario still represents an overparameterized system, and an in-depth examination can provide insights on more intricate situations. Our principal conclusions are encapsulated as follows:

\textbf{Effect of Initialization Scale:} Within the context of single-neuron target recovery, we experimentally demonstrate that  smaller initialization scales are conducive to enhanced generalization.

\textbf{Effect of Randomness:} Randomness retains its significance even as the initialization scale nears zero; we pinpoint a critical variable, termed the "initial imbalance ratio", which serves as a determinant of the training dynamics and generalization error.

\textbf{Effect of Sample Size:} Our empirical results highlight two critical thresholds in sample size—the optimistic sample size and the separation sample size—that align with theoretical forecasts by \citet{zhang2023optimistic,zhang2023structure}. Specifically, we empirically establish that:
   \begin{enumerate}[(i)]
       \item Below optimistic sample size, the model cannot recover the target function.
       \item At optimistic sample size, a zero-measure subset of initialization can recover the target function.
       \item Once the sample size reaches the separation sample size, there exists a non-zero probability that certain combinations of initialization and sampling will successfully recover the target function.
       \item When sample size equals the number of parameters, all small-scale initialization can recover the target function.
   \end{enumerate}

\section{Related works}

The single-neuron fitting problem has been extensively studied, with various works investigating the convergence properties of networks in both exactly parameterized and overparameterized settings~\citep{yehudai2020learning, vardi2021learning, xu2023over, Vempala2018PolynomialCO}. These works have established results on the convergence rates and conditions for neural networks, laying a theoretical groundwork for discussions on generalization. When it comes to generalization, several studies have derived polynomial generalization bounds~\citep{pmlr-v151-wu22c,frei2020agnostic}, while others have presented theoretical results of implicit regularization~\citep{chistikov2024learning, oymak2019overparameterized, safran2022effective}. However, these analyses are typically restricted to the exactly parameterized setting or are specific to the ReLU activation function. In this paper, we empirically examine generalization enigmas in overparameterized networks with analytic activation functions. We demonstrate that perfect generalization is attainable for a certain sample size within the single-neuron target framework, offering a more nuanced characterization of generalization than the polynomial generalization bounds previously reported.


Recent theoretical advancements by~\citet{zhang2023optimistic} and~\citet{zhang2022linear} introduced an optimistic estimate framework for general nonlinear models, suggesting that above a certain "optimistic sample size," some global minima become locally linearly stable, thereby allowing initializations close to these points to converge to stable solutions. Furthermore, \citet{zhang2023structure} delved into the branch structure of global minima in two-layer neural networks, defining a "separation sample size." Despite the theoretical importance of these findings, empirical validation has been limited. Our study aims to bridge this gap by providing a systematic empirical investigation of how initialization and sample size influence the actual dynamics and convergence outcomes in neural network models.

\section{Preliminaries}
\label{sec:Preliminaries}
\subsection{Notations}
  

In this paper, we investigate a two-layer fully connected neural network represented by $f_{\vtheta}(\vx)=\sum_{i=1}^m a_i\sigma(\vw_i^{\top} \vx)$, where $\vx \in \mathbb{R}^d$ and $\vtheta=(a_1,\vw_1,a_2,\vw_2,\ldots, a_m,\vw_m) \in \mathbb{R}^{(d+1)m}$. The function $\sigma:$ $\sR\to \sR$ denotes the activation function, and $m$ represents the width of the network. 
The target function we aim to approximate is a single-neuron function $f^*(\vx)=a_0\sigma(\vw_0^{\top}\vx)$. The dataset $(\vx_i,y_i)_{i=1}^n$ is generated by sampling from the target function $f^*$, that is, $y_i=f^*(\vx_i)$ for $i=1,2,\ldots,n$. We define the loss function as $\ell(\vtheta)=\frac{1}{2}\sum_{i=1}^n(f_{\vtheta}(\vx_i)-y_i)^2$. 





\subsection{Optimistic sample size and separation sample size}

The \emph{target set}, denoted by $Q^*$, is defined as the set of parameters that achieve perfect generalization:
\begin{equation*}
    Q^* := \{\vtheta \mid f_{\vtheta}(\vx) = f^*(\vx), \forall \vx \in \mathbb{R}^d \}.
\end{equation*}

\cite{zhang2023structure} classified $Q^*$ into several affine subspaces for a two-layer neural network without a bias term. We illustrate this with Example~\ref{example1}, where $Q^*$ is the union of two affine spaces. 

\begin{example}
Consider a neural network model $f_{\vtheta}(x) = a_1\tanh(w_1x+b_1) + a_2\tanh(w_2x+b_2)$, where $x\in\sR$. Let the target function be $f^*(x) = a_0\tanh(w_0x+b_0)$. We define:
\begin{align*}
    Q^1:&= \left\{\vtheta \mid (w_1,b_1)=(w_2,b_2)=(w_0,b_0),a_1 + a_2 = a_0\right\} \\
Q^2 :&= \left\{\vtheta \mid (w_1,b_1) = (w_0,b_0), (w_2,b_2)\neq (w_0,b_0), a_1 = a_0,  a_2 = 0\right\}
\end{align*}
Treating parameters symmetric about the origin as identical and the interchange of the two neurons as identical, we have $Q^1 \cup Q^2 = Q^*$(See Figure~\ref{plot Q n=6} for geometric structure of $Q^1$ and $Q^2$).
\label{example1}
\end{example}


\begin{definition}[Separation of $Q^k$]
A set $Q^k$ is said to be separated if there exists an open neighborhood $M$ around $Q^k$ such that $M \cap \ell^{-1}(0) = Q^k \cap \ell^{-1}(0)$, where $\ell^{-1}(0)$ is the set of global minima. The minimum number of samples required for $Q^k$ to be separated is referred to as the "separation sample size".
\label{def:separation}
\end{definition}

Definition~\ref{def:separation} introduces the concept of separation and separation sample size. \cite{zhang2023structure} proves that (i) the separation sample sizes of $Q^1$ and $Q^2$ are $4$ and $5$, respectively, and (ii) when $n=6$, $Q^*=\ell^{-1}(0)$.


\cite{zhang2023optimistic} proposed the concept of the "optimistic sample size," which determines the minimum number of samples required to achieve zero generalization error of a target function. In Example~\ref{example1}, the optimistic sample size is $3$. Table~\ref{table 1} summarizes the various sample sizes of Example~\ref{example1}.


\begin{table}[ht]
\begin{center}
\begin{tabular}{cc} 
\toprule
sample size $n$ &  Name\\ 
\midrule
$n=3$ & optimistic sample size \\ 
$n=4$ & separation sample size of $Q^2$ \\ 
$n=5$ & separation sample size of $Q^1$\\
$n=6$ & $Q^*=\ell^{-1}(0)$\\
\bottomrule
\end{tabular}
\end{center}
\caption{Different sample sizes in Example~\ref{example1}\citep{zhang2023optimistic,zhang2023structure}
}
\label{table 1}
\end{table}

\subsection{Experimental setup}
Our methodology involves sampling a set of data points from the target function $f^*$ and train the network $f_{\vtheta}$ until the parameters converge to $\vtheta_{\infty}$. To evaluate the training effectiveness, we measure the $L_2$ distance between $f^*$ and the learned function $f_{\vtheta_{\infty}}$. The ideal outcome is that $f^* = f_{\vtheta_{\infty}}$, a condition we term "recovery" or "perfect generalization". 

For training, we employ gradient descent with the update rule $\vtheta_{n+1} = \vtheta_{n} - \eta \nabla \ell(\vtheta_n)$, using a fixed learning rate $\eta$. Parameters are initialized according to a Gaussian distribution with mean vector $\mathbf{0}$ and covariance matrix $\sigma I$, where $I$ denotes the identity matrix. The standard deviation $\sigma$ is referred to as the "initialization scale". We utilize a random seed to generate the Gaussian distribution, with each seed uniquely identified by a corresponding number.

\section{Effect of initialization scale}

Our experimental findings suggest a relationship between a small initialization scale of network's parameters and a lower generalization error. Figures \ref{n=2} to \ref{n=6} depict the generalization error across various initialization. The scale of initialization is represented by $\sigma$ on the x-axis. For each initialization scale, we generate initialization using $100$ distinct random seeds. The sample points are fixed, evenly spaced over the interval $[-2,2]$. We observe that, regardless of the value of $n$, the generalization error tends to increase with the sample size, with this trend being particularly noticeable for $n=3$ and $n=4$. In Figure~\ref{scale2}, we use samples drawn independently from an standard Gaussian distribution. Both samples and initialization are generated using $50$ random seeds, and the generalization error is calculated as the average over these $50$ trials. Results of Figure~\ref{scale2} indicate that, statistically, the generalization error decreases as the sample size grows, indicating that larger datasets are conducive to better generalization. Moreover, this reduction in error is more significant at smaller scales of initialization, highlighting the benefits of smaller initializations for achieving improved generalization.

\begin{figure}[htbp]
\centering
\subfigure[fixed sampling, $n=2$]{
\includegraphics[width=0.27\linewidth]{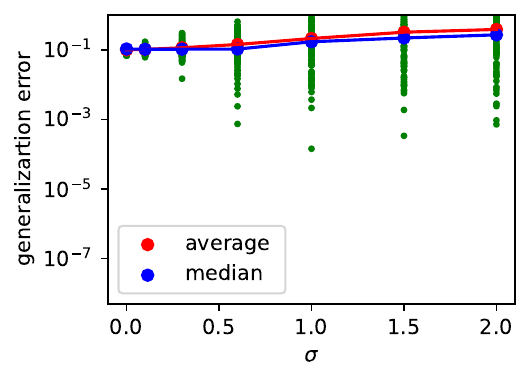}
\label{n=2}}
\subfigure[fixed sampling, $n=3$]{
\includegraphics[width=0.27\linewidth]{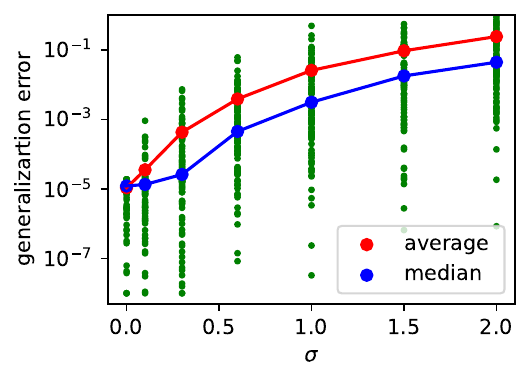} \label{n=3}}
\subfigure[fixed sampling, $n=4$]{
\includegraphics[width=0.27\linewidth]{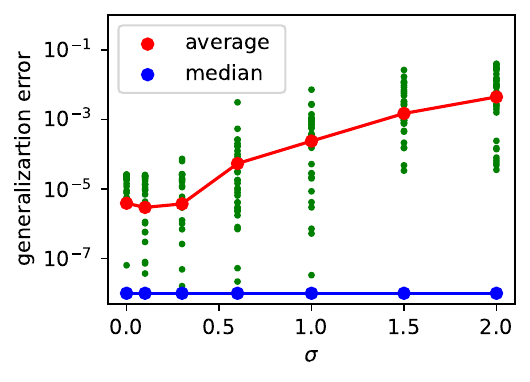} \label{n=4}}
\subfigure[fixed sampling, $n=5$]{
\includegraphics[width=0.27\linewidth]{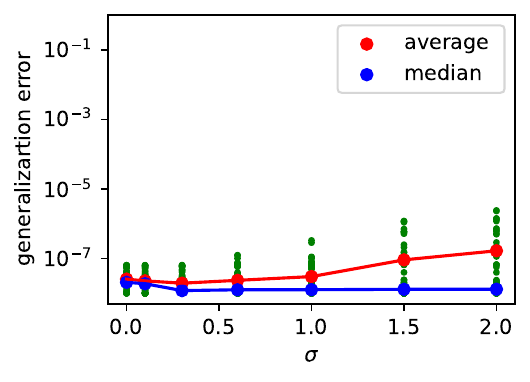} \label{n=5}}
\subfigure[fixed sampling, $n=6$]{
\includegraphics[width=0.27\linewidth]{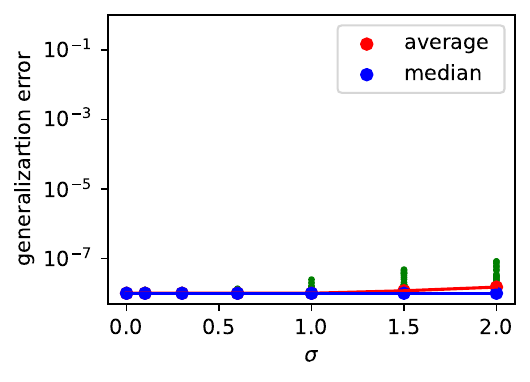} \label{n=6}}
\subfigure[random sampling]{
\includegraphics[width=0.27\linewidth]{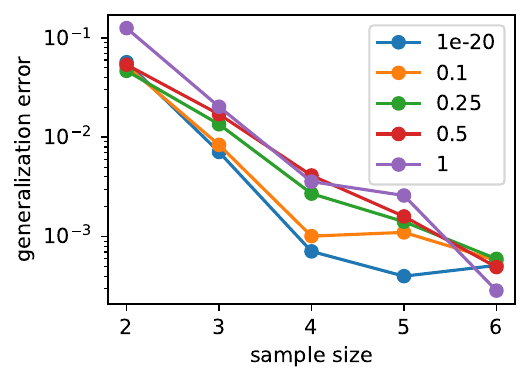}
\label{scale2}}
\caption{
The network and target function correspond to Example~\ref{example1}. Here, $n$ represents the sample size. For Figures~\ref{n=2} through \ref{n=6}, samples were evenly spaced on the interval $[-2, 2]$. In Figure~\ref{scale2}, the dataset $\{(x_i,y_i)\}_{i=1}^n$ is such that $y_i=f^*(x_i)$, with the $\{x_i\}_{i=1}^n$ being independently and identically distributed according to a standard Gaussian distribution. For each combination of initialization scale and sample size, we conducted $50$ trials with different seeds to generate data points and parameter initializations. The reported generalization error is the average over these trials. Curve legends indicate the initialization scale.
}
\label{scale effects}
\end{figure}

\section{Effect of randomness of initialization}
\label{sec: effect of randomness}
\subsection{Initialization and trajectory of parameters}

Previously, we empirically demonstrated that small-scale initialization enhances the generalization. This section delves into the dynamics of gradient flow under small initialization.

\begin{theorem}
Consider the gradient flow governed by the differential equation
\begin{equation}
\begin{aligned}
    &\frac{d\vtheta}{dt} = -\nabla \ell(\vtheta(t))
    ,\vtheta(0) = \vtheta_{0},
\end{aligned}
\label{eq:main1}
\end{equation}
where $\ell(\vtheta) = \frac{1}{2}\sum_{i=1}^n (f_{\vtheta}(\vx_i) - y_i)^2$ for $(\vx_i,y_i) \in \mathbb{R}^d \times \mathbb{R}$, with the model $f_{\vtheta}(\vx) = \sum_{k=1}^m a_k\sigma(\vw_k^\top \vx)$, and the parameter vector $\vtheta = (a_1, \vw_1, \ldots, a_m, \vw_m) \in \mathbb{R}^{m(d+1)}$. The solution to \eqref{eq:main1} is denoted by $\phi(\vtheta_{0},t)$. Define $\vgamma := \sum_{i=1}^n y_i\vx_i$, $C_i(\vtheta) := a_i\|\vgamma\|_2 + \vw_i^{\top} \vgamma$, and $\vC(\vtheta) := (C_1(\vtheta), \ldots, C_m(\vtheta))$.
Assume the following conditions:
\begin{enumerate}[(i)]
    \item  $\sigma(x)$ is twice continuously differentiable on $\mathbb{R}$, $\sigma(0) = 0$, and $\sigma'(0) \neq 0$.

    \item  $\vgamma \neq \mathbf{0}$.
\end{enumerate}

Under these assumptions, the following statements hold: 
\begin{enumerate}[(i)]
\item For any $t \in \mathbb{R}$ and $\vtheta \in \mathbb{R}^{m(d+1)}$, the limit $h(\vtheta,t) := \lim_{\alpha \to 0} \phi(\alpha \vtheta, t + \frac{1}{\|\vgamma\|_2}\log\frac{1}{\alpha})$ exists.

\item The function $h(\vtheta_0,t)$ is  determined by $\vC(\vtheta_0)$. That is, if $\vC(\vtheta_1) = \vC(\vtheta_2)$, then $h(\vtheta_1,t) = h(\vtheta_2,t)$ for all $t$.

\item If $\vC(\vtheta_0) \neq \mathbf{0}$, then the trajectory $T_{\vtheta_0} := \{h(\vtheta_0,t) : t \in \mathbb{R}\}$ is determined by $\frac{\vC(\vtheta_0)}{\|\vC(\vtheta_0)\|_2}$. That is, if $\vC(\vtheta_1) = \vC(\vtheta_2)$, then $T_{\vtheta_1} = T_{\vtheta_2}$.
\end{enumerate}
\label{theorem1}
\end{theorem}

The proof of Theorem~\ref{theorem1} is in the Appendix~\ref{Proof of theorem1}. A  more general result for dynamic systems is stated in Theorem~\ref{theorem3} in Appendix\ref{Proof of theorem1}. To intuitively understand Theorem~\ref{theorem1}, we consider the linearization of Equation~\eqref{eq:main1} at the origin. Assumption (i) of Theorem~\ref{theorem1} ensures that $\nabla \ell(\mathbf{0}) = \mathbf{0}$, allowing us to approximate $-\nabla \ell(\vtheta(t))$ by $-\text{Hess}(\ell(\mathbf{0}))\vtheta$ when $\|\vtheta\|_2$ is small. Under this linear approximation, the solution to Equation~\eqref{eq:main1} can be expressed as $\vtheta(t) \approx e^{-\text{Hess}(\ell(\mathbf{0}))t}\vtheta_0$. When the norm $\|\vtheta_0\|_2$ is sufficiently small, a large $t$ is required for $\vtheta(t)$ to move significantly away from the origin. In such cases, the largest eigenvalue of $-\text{Hess}(\ell(\mathbf{0}))$, denoted $\mu_1$, becomes dominant, leading to
\[
\vtheta(t) \approx e^{\mu_1 t}\sum_{i=1}^k (\vtheta_0^\top \vv_i)\vv_i,
\]
where $\{\vv_1, \ldots, \vv_k\}$ is the orthonormal basis of the eigenspace corresponding to $\mu_1$. The evolution of $\vtheta(t)$ is thus determined by the coefficients $\{\vtheta_0^\top \vv_i\}_{i=1}^k$, which are encapsulated in the vector $\vC(\vtheta_0)$ defined in Theorem~\ref{theorem1}. 

In the context of two-layer neural networks, $\vC(\vtheta_0) = (C_1(\vtheta_0), C_2(\vtheta_0), \ldots, C_m(\vtheta_0))$ provides an insightful interpretation. The expression $\vtheta(t) \approx e^{\mu_1 t}\sum_{i=1}^k (\vtheta_0^\top \vv_i)\vv_i$ suggests that for $i = 1, 2, \ldots, m$, the following approximations hold:
\begin{equation}
\begin{aligned}
   & a_{i}(t) \approx \frac{C_{i}(\vtheta_0)}{2 \|\vgamma\|_2}e^{\|\vgamma\|_2 t}, 
    \vw_{i}(t) \approx \frac{C_{i}(\vtheta_0)\vgamma}{2\|\vgamma\|_2^2}e^{\|\vgamma\|_2 t},
\end{aligned}
\label{eq:approx}
\end{equation}
where $\vgamma$ is a vector determined by the data. Equation \eqref{eq:approx} indicates that the direction of the vector $(a_i, \vw_i)$ is consistent across all neurons, characterized by $\vgamma$. This observation is in line with the findings of \cite{zhou2022towards}, which show that neural networks with small initial weights tend to have input weights of hidden neurons aligning along certain data-determined directions. Moreover, Equation \eqref{eq:approx} indicates that the magnitude of $(a_i, \vw_i)$ is determined by $C_i(\vtheta_0)$. Thus, the vector $\vC(\vtheta_0)$, representing the initial magnitudes of all neurons, determines early evolution of the parameters. The normalized vector $\frac{\vC(\vtheta_0)}{\|\vC(\vtheta_0)\|_2}$, representing the initial relative magnitudes of all neurons, determines trajectory of the parameters. Due to this, we refer $\frac{\vC(\vtheta_0)}{\|\vC(\vtheta_0)\|_2}$ as "initial imbalance ratio".


To corroborate the theoretical insights posited by Theorem~\ref{theorem1}, we conducted a series of experiments. We chose a model of the form $f_{\vtheta}(x)=a_1\tanh(w_1x+b_1)+a_2\tanh(w_2x+b_2)$, with the target function defined as $f^*(x)=\tanh(x+1)$.
According to Theorem~\ref{theorem1}, under small initialization, the parameter trajectory is determined by the normalized vector $\left(\frac{C_1(\vtheta_0)}{\sqrt{C_1^2(\vtheta_0)+C_2^2(\vtheta_0)}}, \frac{C_2(\vtheta_0
)}{\sqrt{C_1^2(\vtheta_0)+C_2^2(\vtheta_0)}}\right)$. Given that $\sigma(x) = \tanh(x)$ is an odd function, the sign inversion of both $C_1(\vtheta_0)$ and $C_2(\vtheta_0)$ leads to a symmetric trajectory about the origin, which we consider equivalent. Therefore, the trajectory is effectively characterized by the ratio $\frac{C_1(\vtheta_0)}{C_2(\vtheta_0)}$, denoted by $c(\vtheta_0):=\frac{C_1(\vtheta_0)}{C_2(\vtheta_0)}$. For simplicity, we will henceforth denote $c(\vtheta_0)$ by $c$.

Figure~\ref{effect of c} shows the training results across five trials. Each trial used a different initialization scale and random seed to generate Gaussian distribution but kept the ratio $c=0.5$ across five trials by scaling the initialization of second neuron. The results demonstrate that both the loss and the parameter trajectories were consistent across all trials, lending strong empirical support to the theoretical prediction that the trajectory is governed by the ratio $c$ under small initialization.

 
\begin{figure}[htbp]
\centering
\subfigure[]{
\includegraphics[width=0.292\linewidth]{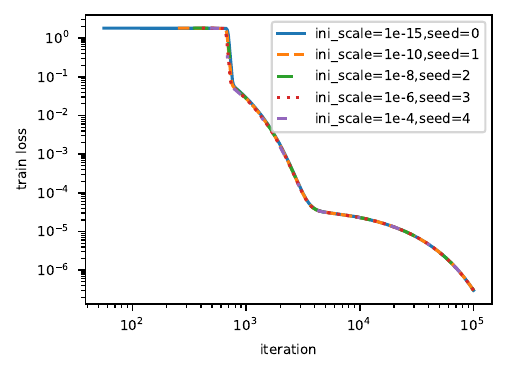}
\label{c and loss}}
\subfigure[]{
\includegraphics[width=0.468\linewidth]{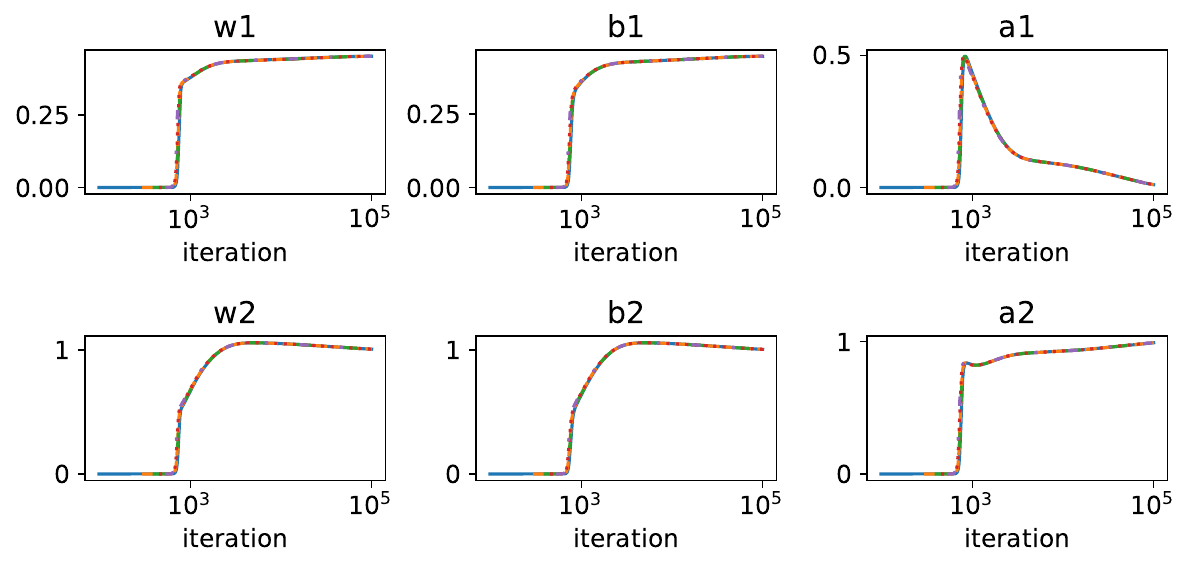}
\label{c and para}}

\caption{
The network and target function correspond to Example~\ref{example1}. We trained the network across five trials, each utilizing an evenly spaced $6$ data points within the interval $[-2,2]$. Distinct initialization seeds and scales were used for each trial, but by scaling the initial parameters of the second neuron, we keep $c=0.5$ across all trials. To align the curves, we applied translations based on distances calculated by Theorem~\ref{theorem1}.
}
\label{effect of c}
\end{figure}

\subsection{Initialization and convergence point}

As the scale of initialization approaches zero, the parameters' trajectory has a limit. 
A natural question arises: does the convergence point of the parameters also tend towards a limit as the initialization scale becomes infinitesimally small? We affirmatively address this question in Theorem~\ref{theorem2}.

\begin{theorem}
    Under the notations and assumptions of Theorem~\ref{theorem1}, and assuming that $\sigma(x)$ is analytic, then:
    \begin{enumerate}[(i)]
        \item The limit $h(\vtheta,t) := \lim_{\alpha \to 0} \phi(\alpha \vtheta, t + \frac{1}{\|\vgamma\|_2}\log\frac{1}{\alpha})$ exists.
        \item For any $\vtheta$, if the set $\{h(\vtheta,t) : t \geq 0\}$ is bounded, then the limit $\lim_{t \to \infty} h(\vtheta,t)$ exists and is determined by the normalized vector $\frac{\vC(\vtheta)}{\|\vC(\vtheta)\|_2}$. Specifically, if $\frac{\vC(\vtheta_1)}{\|\vC(\vtheta_1)\|_2} = \frac{\vC(\vtheta_2)}{\|\vC(\vtheta_2)\|_2}$, then $\lim_{t \to \infty} h(\vtheta_1,t) = \lim_{t \to \infty} h(\vtheta_2,t)$.
        \item If $\lim_{t \to \infty} h(\vtheta_0,t)$ exists and is not a saddle point of $\ell(\vtheta)$, then 
        \[
        \lim_{t \to \infty} h(\vtheta_0,t) = \lim_{\alpha \to 0} \lim_{t \to \infty} \phi(\alpha \vtheta_0, t + \frac{1}{\|\vgamma\|_2}\log\frac{1}{\alpha}).
        \]
        Additionally, the limit $\lim_{t \to \infty} h(\vtheta,t)$ is continuous at $\vtheta_0$. 
    \end{enumerate}
    \label{theorem2}
\end{theorem}

The proof of Theorem~\ref{theorem2} is presented in the Appendix~\ref{Proof of Theorem2}. We conduct experiments   when$f_{\vtheta}(x) = a_1\tanh(w_1x+b_1) + a_2\tanh(w_2x+b_2)$ and $f^*(x) = \tanh(x+1)$. Literature suggests that convergence to a saddle point is rare \citep{panageas2019first} for gradient flow. Moreover, in neural network experiments, divergence of $\vtheta(t)$ to infinity is seldom observed. Thus, the conditions of (ii) and (iii) are typically met. The conclusion (iii) of Theorem~\ref{theorem2} affirms that for a  small initialization scale, the convergence point of the parameters can be brought arbitrarily close to $\lim_{t \to \infty} h(\vtheta_0, t)$. The conclusion of (ii) implies that $\lim_{t \to \infty} h(\vtheta_0, t)$ is determined by $\frac{\vC}{\|\vC\|_2}$. Therefore, under sufficiently small initialization, the vector $\frac{\vC}{\|\vC\|_2}$ almost determines the final convergence point of the parameters. Define $c := C_1/C_2$ and $\Tilde{c} := \min\{|c|, |\frac{1}{c}|\}$. 
Networks initialized with $c$ and $\frac{1}{c}$ are identical upon permuting the neurons and network initialized with $c$ and $-c$ yields trajectories symmetric about the origin. Hence, $\Tilde{c}$ effectively encompasses all cases of initialization by accounting for symmetry. 

In Figure~\ref{c and converging}, we investigate the impact of initialization on the convergence point under small initialization. Figure~\ref{plot Q n=6} visualizes the target set $Q^*$ alongside the convergence points for various $\Tilde{c}$. The line is $Q^1$, while the surface are $Q^2$. These simulations, run over $10^6$ iterations with a learning rate of $0.05$ and a sample size of $6$, reveal two significant observations:
\begin{enumerate}[(i)]
    \item The parameters consistently converge to $Q^*$ across different initialization, corroborating the theoretical results in Table~\ref{table 1} that for $n \geq 6$, the global minimum coincides with $Q^*$.
    \item The convergence points lie on a one-dimensional manifold parameterized by $\Tilde{c}$. Notably, as $\Tilde{c}$ nears $1$, the convergence point moves closer to $Q^1$. 
\end{enumerate}

Figure~\ref{c and Q1 Q2} further illustrates the convergence behavior of the network's parameters for diverse initialization, focusing on convergence to points $Q^1$ or $Q^2$. The results demonstrate that convergence occurs at $Q^1$ when $c$ approximates $1$. For $c$ values significantly divergent from $1$, convergence at $Q^2$ becomes more likely, with the demarcation at $c=1.35$ and $c=0.74$.

In the small initialization dynamics of a neural network, ratio $c$ governs relative magnitude of two neurons at initial stage. A $c$ close to $1$ suggests similar magnitudes for both neurons, leading to convergence at $Q^1$, where each neuron contributes to the output function. Conversely, a substantially larger $c$ implies that the first neuron's magnitude predominates, resulting in convergence at $Q^2$, where the second neuron's output contribution is zero. The two extremes, $c=1$ and $c=+\infty$, demonstrate this: for $c=1$, the ratios $\frac{a_{1}}{a_{2}}$, $\frac{w_{1}}{w_{2}}$, and $\frac{b_{1}}{b_{2}}$ remain constant at $1$ for all $t$, converging to $Q^1$. For $c=+\infty$, we have $a_{2}=w_{2}=b_{2}=0$ for all $t$, resulting in convergence at $Q^2$.

\begin{figure}[htbp]
\centering
\subfigure[]{
\includegraphics[width=0.4\linewidth]{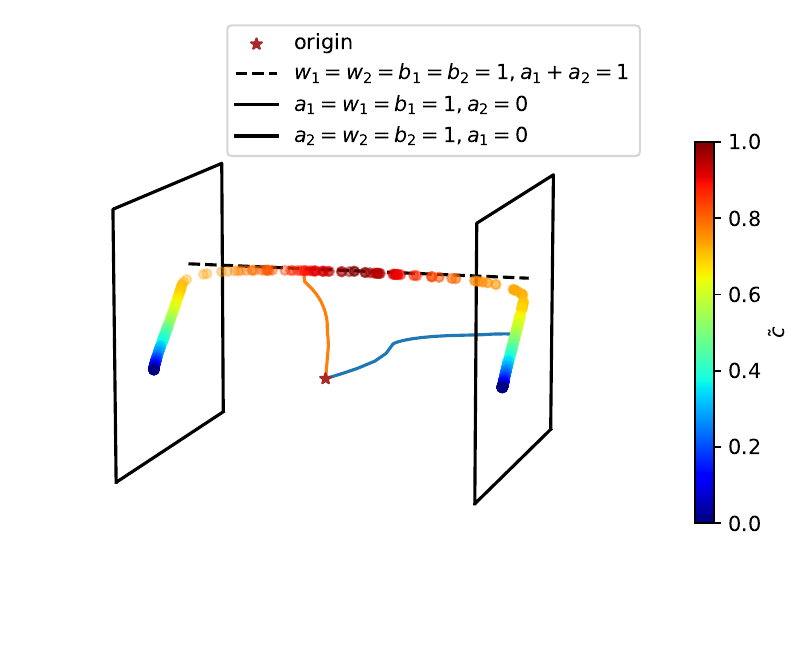}
\label{plot Q n=6}}
\subfigure[]{
\includegraphics[width=0.4\linewidth]{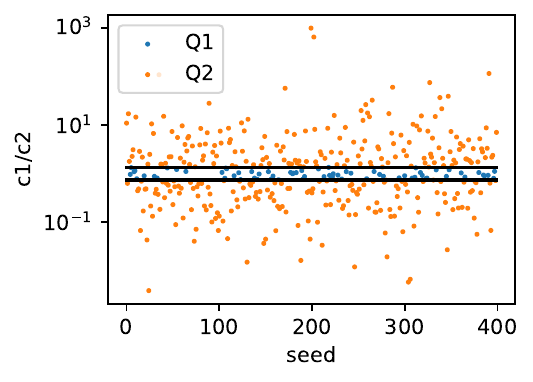}
\label{c and Q1 Q2}}

\caption{
The network and target function correspond to Example~\ref{example1} with a sample size of $6$ and an initialization scale of $10^{-8}$. We utilized $400$ random seeds to initialize the parameters. Figure~\ref{plot Q n=6} shows the convergence points and the structures of $Q^1$ and $Q^2$, along with the origin and two exemplary training trajectories. The dashed line is $Q^1$ and the affine surface is $Q^2$. Figure~\ref{c and Q1 Q2} presents the convergence results using seeds $0$-$400$, where blue and orange represent convergence to $Q^1$ and $Q^2$, respectively. The x-axis denotes the seed index, and the y-axis measures the absolute value of the ratio $C_{1}/C_{2}$. The two black horizontal lines mark the ratios at $y=1.35$ and $y=0.74$.
}
\label{c and converging}
\end{figure}

\section{Effect of sample size}

This section delves into the impact of training sample size on the network’s ability to achieve recovery. In Figure~\ref{c and genloss}, we present the generalization error under various combinations of initialization and sample sizes. The network in question is defined as $f_{\vtheta}(x)=a_1\tanh(w_1x+b_1)+a_2\tanh(w_2x+b_2)$, with the target function being $f^*(x)=\tanh(x+1)$. For small initialization scale, the ratio $\tilde{c}=\min\{|c|,|\frac{1}{c}|\}$ adequately represents all initialization. 
When $n=2$, the network fails to recover the target function for any initialization, in line with the optimistic sample size theory~\citep{zhang2023optimistic}. To elucidate, a single-neuron target encompasses $3$ parameters. With only two samples, an infinite number of single-neuron targets could fit, preventing the network from identifying the desired target.

For $n=3$ (see Figure~\ref{c and genloss,n=2}), recovery is feasible solely when $\tilde{c}=0$ or $\tilde{c}=1$. Initially, with $n=3$, neither $Q^1$ nor $Q^2$ is separated (see Table~\ref{table 1}). The target sets $Q^1$ and $Q^2$, enveloped by global minima, 
lead the network to likely converge to these global minima rather than the target set. This accounts for the lack of recovery when $\tilde{c}\in (0,1)$. Nonetheless, fortuitous scenarios occur. When $\tilde{c}=1$, the ratios $\frac{a_1}{a_2}=\frac{w_1}{w_2}=\frac{b_1}{b_2}=1$ hold during training, simplifying the two-neuron network to a single-neuron model, reducing six parameters to three. In such instances, three samples implies that the global minimum equals the target set~\citep{zhang2023structure}, facilitating recovery when $\tilde{c}=1$. The case of $\Tilde{c}=0$ is similar.

For $n=4$ (see Figure~\ref{c and genloss, n=4} and Figure~\ref{c and genloss, recover at Q^2}), recovery is attainable with a positive probability of sampling and initialization. Meanwhile, with some samples, recovery remains unattainable for all initialization. Table~\ref{table 1} indicates that for $n=4$, $Q^2$ is separated, whereas $Q^1$ is not. Our findings corroborate that once $Q^2$ is separated, convergence to it is plausible with a positive probability. Additionally, we demonstrate the existence of samples for which no small-scale initialization leads to convergence to the target set.

For $n=5$, some samples enable the network to recover for all small initialization, while others do not. Table~\ref{table 1} shows that for $n=5$, both $Q^1$ and $Q^2$ are separated. Our experiments suggest that once $Q^1$ and $Q^2$ are separated, under some samples, recovery is achievable across all initialization. Moreover, at $n=5$, certain global minima with non-zero generalization error may be encountered during training with specific samples. For $n=6$, the network recovers for all small initialization, as all global minima are associated with zero generalization error (see Table~\ref{table 1}).

Using an analogy, we can liken the process of recovery to archery. The size of the training sample dictates the structural configuration of the target. Once the sample size surpasses the threshold known as the separation sample size, certain sections of the target become exposed and unobscured by any other global minima, rendering them accessible with a non-zero probability. Concurrently, there are certain shortcuts that facilitate reaching the target more directly. Specifically, when $\tilde{c}=0$ or $\tilde{c}=1$, the network is capable of hitting the target at the so-called optimistic sample size, even if $Q^1$ and $Q^2$ remain concealed by global minima. In essence, the sample size molds the target's architecture, while the initialization steers the direction of the "shot".

\begin{figure}[htbp]
\centering
\subfigure[$n=2$]{
\includegraphics[width=0.27\linewidth]{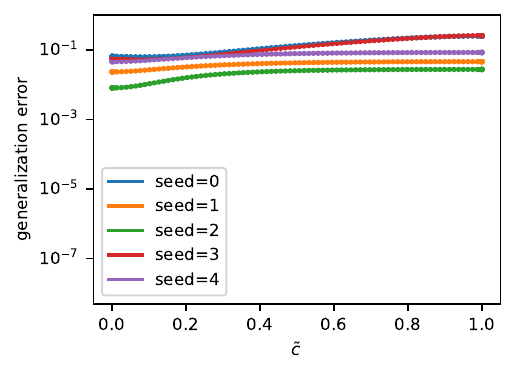}
\label{c and genloss,n=2}}
\subfigure[$n=3$]{
\includegraphics[width=0.27\linewidth]{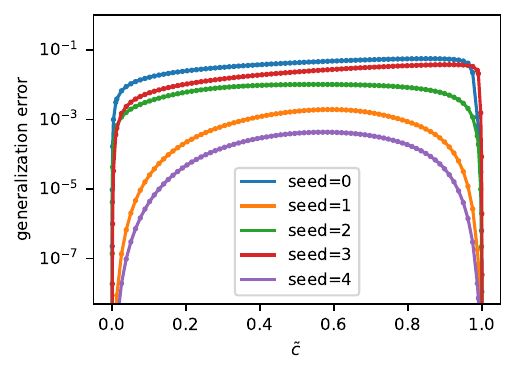}
\label{c and genloss, n=3}}
\subfigure[$n=4$]{
\includegraphics[width=0.27\linewidth]{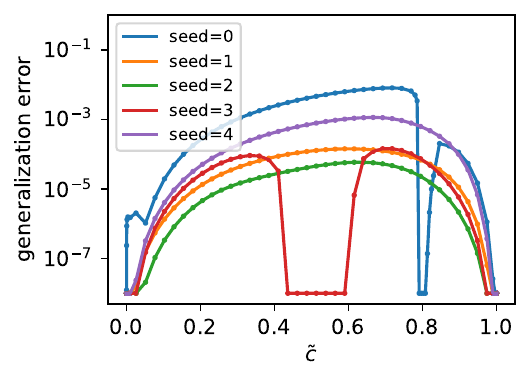}
\label{c and genloss, n=4}}
\subfigure[$n=5$]{
\includegraphics[width=0.27\linewidth]{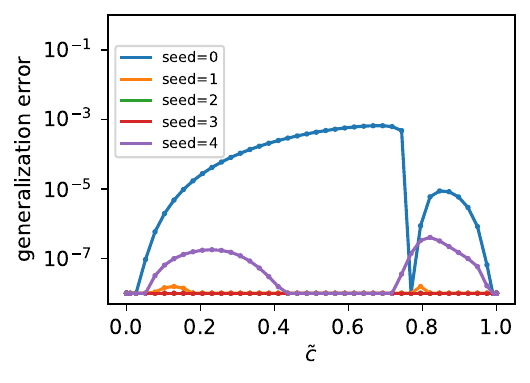}
\label{c and genloss, n=5}}
\subfigure[$n=6$]{
\includegraphics[width=0.27\linewidth]{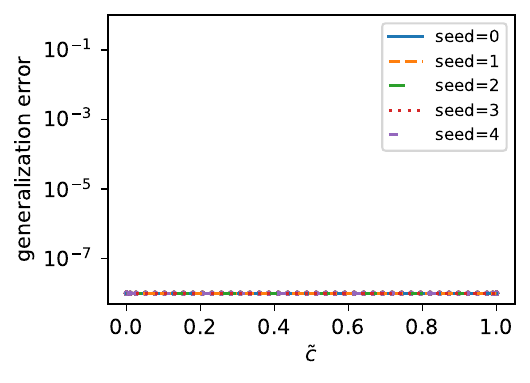}
\label{c and genloss, n=6}}
\subfigure[$n=4$, seed=$3$]{
\includegraphics[width=0.27\linewidth]{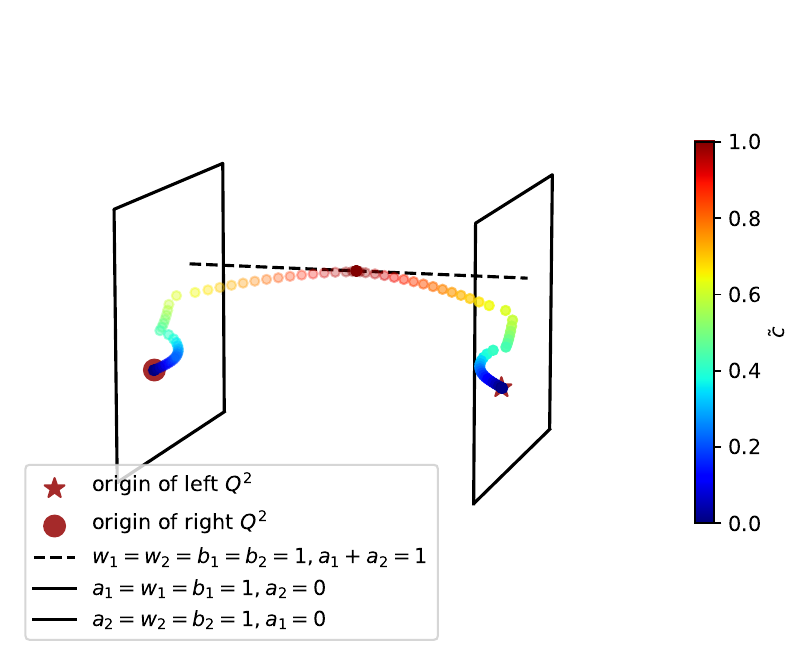}
\label{c and genloss, recover at Q^2}}
\caption{
The network $f_{\vtheta}(x)$ and target function $f^*$ correspond to Example~\ref{example1}. The samples  $\{(x_i,y_i)\}_{i=1}^n$, where $y_i=f^*(x_i)$, is obtained by drawing $\{x_i\}_{i=1}^n$ from a standard Gaussian distribution. Five random seeds were used to generate the samples. Generalization errors below $10^{-8}$ are considered as successful recovery and identified with $10^{-8}$. Figure~\ref{c and genloss, recover at Q^2} depicts the convergence point for $n=4$ with samples generated by seed $3$. The dashed line is $Q^1$ and the affine surface is $Q^2$. All experiments were initialized with a scale of $10^{-20}$.
}
\label{c and genloss}
\end{figure}


\section{Extension to multi-neuron networks}
The insights from the two-neuron, two-layer neural network analysis can be generalized to networks with multiple neurons. Our experiments on a neural network with a width of $1000$ and the activation function $\sigma(x)=\frac{x}{1+x^2}$ support this generalization. As depicted in Figure~\ref{multi neuron, effect of c}, the experiments reveal that under small initialization conditions, the parameter trajectories conform to the set $\frac{\vC}{\|\vC\|_2}$, as postulated in Theorem~\ref{theorem1}.

Figure~\ref{multilayer-weight} shows that in a two-layer neural network approximating a single-neuron target function, only a subset of neurons develop substantial weights and become key contributors to output function. These neurons are distinguished by having the largest $C_k$ values, aligning with earlier experimental observations of two-neuron network that neurons with higher $C_k$ values tend to have greater magnitudes. In Figure~\ref{multilayer-output function}, we note that the generalization error is significantly low in an overparameterized network. This is partly attributed to the phenomenon in Figure~\ref{multilayer-weight}, where most neurons possess minimal weights compared to the neuron with the largest weight magnitude. Consequently, the neural network operates as if it has fewer active neurons. This decrease in active neuron count effectively reduces the network's complexity and improves its generalization capability. Besides these experiments, we also conducted experiments with higher dimensional input (see Appendix~\ref{Appen:higher dimensional input}).

\begin{figure}[ht]
\centering
\subfigure[]{
\includegraphics[width=0.292\linewidth]{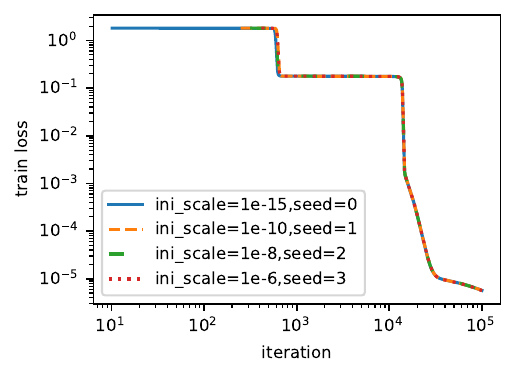}
\label{multi neuron, c and loss}}
\subfigure[]{
\includegraphics[width=0.468\linewidth]{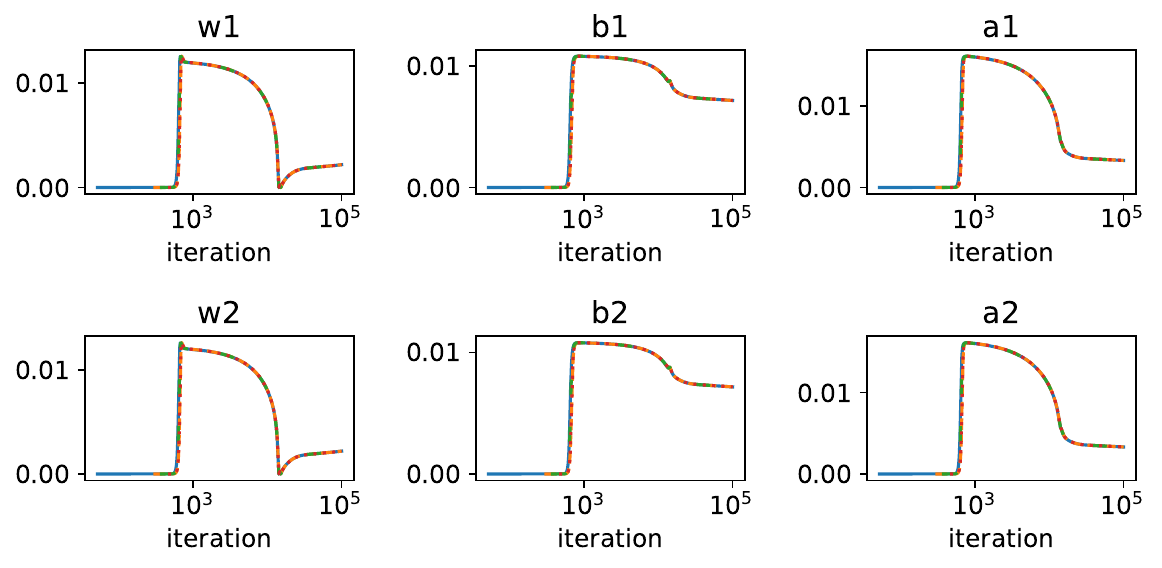}
\label{multi neuron, c and para}}
\caption{
A two-layer neural network with a width of $1000$ and activation function $\sigma(x)=\frac{x}{1+x^2}$ is trained on $6$ evenly spaced data points in the interval $[-2,2]$ with labels given by $y=\tanh(x+1)$. Four trials with varying initialization seeds and scales were conducted. The ratio of initial parameters $C_{i}/C_{1}$ is set to $1.5+0.0015(i-1)$ for each neuron $i=1,2,\ldots,1000$ in all trials. For visualization, curves in Figure~\ref{multi neuron, c and loss} are translated based on distances derived from Theorem~\ref{theorem1}. Figure~\ref{multi neuron, c and para} shows the parameter trajectories for the first two neurons.}
\label{multi neuron, effect of c}
\end{figure}

\begin{figure}[h!]
\centering
\subfigure[]{
\includegraphics[width=0.33\linewidth]{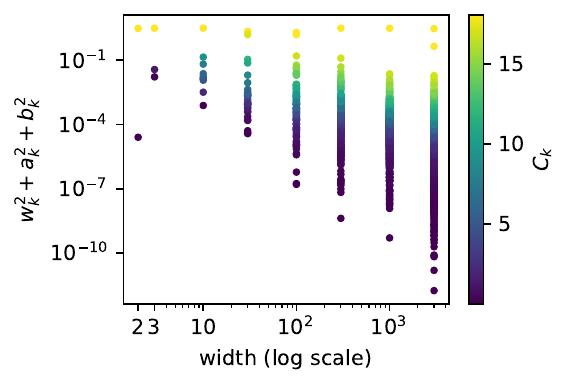}
\label{multilayer-weight}}
\subfigure[]{
\includegraphics[width=0.44\linewidth]{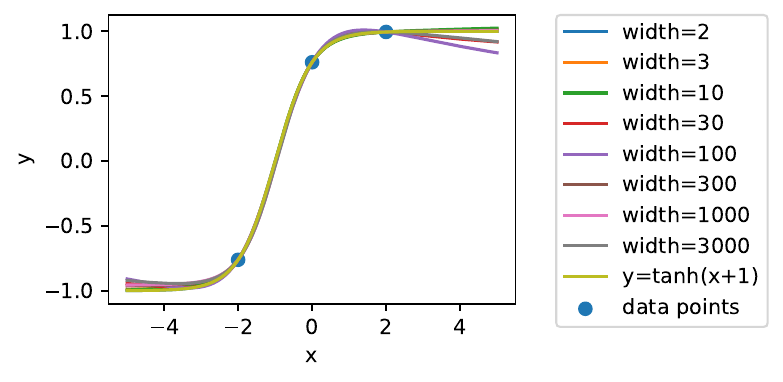}
\label{multilayer-output function}}
\caption{
Visualization of training dynamics and outcomes for a two-layer neural network with variable widths. The networks use the $\tanh(x)$ activation function and are trained to approximate the target function $y=\tanh(x+1)$ using a dataset of $3$ points equally spaced within the interval $[-2,2]$. Figure~\ref{multilayer-weight} shows the magnitude of the weights for individual, with each dot representing a neuron and the dot color indicating the absolute value of the scaling factor $C_{k}$ for the $k$-th neuron. Figure~\ref{multilayer-output function} illustrates the final output functions of the networks with different widths after training.
}
\label{multi neuron recovery}
\end{figure}

\section{Conclusion}
\label{sec:Conclusion}


Our investigation into the learning of single-neuron target functions within two-layer neural networks has elucidated the pivotal influence of initialization scale, randomness, and sample size on achieving perfect generalization. We found that smaller initialization scales and larger sample sizes tend to enhance generalization performance, while the element of randomness plays a significant role in shaping the learning outcomes. By honing in on the concept of perfect generalization, we have simplified the complexity inherent in the generalization puzzle and have empirically validated the existence of both optimistic and separation sample size thresholds. These observations underscore the intricate interplay among initialization, stochastic elements, and sample size in the learning process of neural networks.

We must recognize the limitations imposed by the simplicity of our experimental framework, which was confined to the recovery of a single neuron. Additionally, in Theorem~\ref{theorem2}, we assume convergence to a local minimum instead of  proving it. Further research into the learning behaviors of networks with more complex target functions, as well as the effects of critical points on learning dynamics, represents a compelling direction for future inquiry.


\begin{ack}
This work is sponsored by the National Key R\&D Program of China  Grant No. 2022YFA1008200, the National Natural Science Foundation of China Grant No. 12101402, the Lingang Laboratory Grant No. LG-QS-202202-08, Shanghai Municipal of Science and Technology Major Project No. 2021SHZDZX0102.
\end{ack}

\appendix
\section{Proof of Theorems}
\label{sec:proof of theorems}

\subsection{Proof of Theorem\ref{theorem1}}
\label{Proof of theorem1}
We begin by establishing Theorem~\ref{theorem3}, and then leverage it to validate Theorem~\ref{theorem1}.

Let $\phi(\vtheta_{0}, t)$ denote the solution to the following differential equation \eqref{eq:no loss}:
\begin{equation}
\begin{aligned}
    &\frac{d\vtheta}{dt} = g(\vtheta), \\
    &\vtheta(0) = \vtheta_{0},
\end{aligned}
\label{eq:no loss}
\end{equation}

\begin{theorem}
\label{theorem3}
Assume the conditions:
\begin{itemize}
    \item[(i)] $g(\vtheta)$ is twice continuously differentiable.
    \item[(ii)] $g(\mathbf{0}) = \mathbf{0}$.
    \item[(iii)] $\nabla g(\vtheta)|_{\vtheta=\mathbf{0}}$ is diagonalizable.
    \item[(iv)] The largest eigenvalue of $\nabla g(\vtheta)|_{\vtheta=\mathbf{0}}$ is positive.
\end{itemize}
Denote the solution \ref{eq:no loss} as $\phi(\vtheta_0,t)$.Then, the limit $\lim_{\alpha \to 0} \phi(\alpha \vv_{1} + \vu_{\alpha}, t + \frac{1}{\mu_{1}}\log(\frac{1}{\alpha}))$ exists, and the rate of convergence is $\alpha^{\frac{\mu_{1} - \mu_{2}}{2\mu_{1} - \mu_{2}}}$, where $\mu_{1}$ and $\mu_{2}$ are the largest and second-largest eigenvalues of $\nabla g(\vtheta)|_{\vtheta=\mathbf{0}}$, respectively. $\vv_{1}$ is a vector in the eigenspace corresponding to $\mu_{1}$. The vector $\vu_{\alpha}$ is arbitrary, subject to the conditions:
\begin{itemize}
    \item[(i)] $\vu_{\alpha}$ is orthogonal to the eigenspace of $\mu_{1}$.
    \item[(ii)]$\exists c >0$, such that $\forall \alpha > 0 $, $\|\vu_{\alpha}\|_{2} \leq c\alpha$
\end{itemize}
\end{theorem}

\textbf{Intuition of the Theorem:}\\
Let us denote $\nabla g(\vtheta)$ by $\vJ(t)$. Linearizing the dynamics around the origin, we have:
\begin{equation*}
    \frac{d\vtheta}{dt} = \vJ(0)\vtheta,
\end{equation*}
which yields the linearized solution:
\begin{equation*}
    \vtheta(t) = e^{\vJ(0)t}\vtheta_{0}.
\end{equation*}
When the initial condition $\vtheta_{0}$ is very small, a large $t$ is required to move away from zero. Consequently, the top eigenvalue of $\vJ(0)$ will dominate when computing $e^{\vJ(0)t}$. Hence, only the projection of $\vtheta_{0}$ onto the eigenspace corresponding to $\mu_{1}$ will significantly affect the trajectory of the dynamics.

\textbf{Sketch of the Proof:}\\
We consider an $\epsilon$-ball centered at the origin and determine when the trajectory of $\vtheta(t)$ intersects with this $\epsilon$-ball. On one hand, $\epsilon$ cannot be too small; otherwise, the time $t$ would be small, and the exponential term $e^{\mu_{1}t}$ would not be dominant. On the other hand, $\epsilon$ cannot be too large; otherwise, the linear approximation of $g(\vtheta)$ would not be valid. Therefore, we must choose an appropriate value for $\epsilon$ and analyze the error caused by the two reasons mentioned above.

\begin{proof}[\textbf{Formal Proof of Theorem\ref{theorem3}:}]
Because $\vJ(0)$ is diagonalizable, we can transform the coordinate system such that $\vJ(0)$ becomes a diagonal matrix. Without loss of generality, we assume that $\vJ(0)$ is the diagonal matrix $\diag\{\mu_{1}, \mu_{2}, \ldots, \mu_{d}\}$. We reference Lemma E.3, Lemma E.4, and Lemma E.5 from \cite{li2020towards}, which prove a special case of Theorem \ref{theorem3} where the eigenspace corresponding to $\mu_{1}$ is one-dimensional.

We restate their lemmas, denoting $F(x) = \log(x) - \log(1 + \kappa x)$. Let $T_{\alpha}(r) = \frac{1}{\mu_{1}}(F(r) - F(\alpha))$. Let $R>0$. Since $g(\vtheta)$ is $\mathcal{C}^2$-smooth, there exists $\beta>0$ such that
$
\|\boldsymbol{J}(\boldsymbol{\theta})-\boldsymbol{J}(\boldsymbol{\theta}+\boldsymbol{h})\|_2 \leq \beta\|\boldsymbol{h}\|_2
$, $\text { for all }\|\boldsymbol{\theta}\|_2,\|\boldsymbol{\theta}+\boldsymbol{h}\|_2 \leq R \text {. }$Then we have:

\textbf{Lemma E.3.} For $\boldsymbol{\vtheta}(t) = \phi(\boldsymbol{\vtheta}_0, t)$ with $\|\boldsymbol{\vtheta}_0\|_2 \leq \alpha$ and $t \leq T_\alpha(r)$, it holds that
\[
\|\boldsymbol{\vtheta}(t)\|_2 \leq \frac{1 + \kappa r}{1 + \kappa \alpha} \alpha \cdot e^{\mu_1 t} \leq r.
\]

\textbf{Lemma E.4.} For $\boldsymbol{\vtheta}(t) = \phi(\boldsymbol{\vtheta}_0, t)$ with $\|\boldsymbol{\vtheta}_0\|_2 \leq \alpha$ and $t \leq T_\alpha(r)$, we have
\[
\boldsymbol{\vtheta}(t) = e^{t J(\mathbf{0})} \boldsymbol{\vtheta}_0 + O(r^2).
\]

\textbf{Lemma E.5.} Let $\boldsymbol{\vtheta}(t) = \phi(\boldsymbol{\vtheta}_0, t)$ and $\hat{\boldsymbol{\vtheta}}(t) = \phi(\hat{\boldsymbol{\vtheta}}_0, t)$. If $\max\{\|\boldsymbol{\vtheta}_0\|_2, \|\hat{\boldsymbol{\vtheta}}_0\|_2\} \leq \alpha$, then for $t \leq T_\alpha(r)$,
\[
\|\boldsymbol{\vtheta}(t) - \hat{\boldsymbol{\vtheta}}(t)\|_2 \leq e^{\mu_1 t + \kappa r}\|\boldsymbol{\vtheta}_0 - \hat{\boldsymbol{\vtheta}}_0\|_2.
\]

Where $\kappa = \frac{\beta}{\mu_{1}}$.

Given $\alpha$, we compare $\phi(\alpha \vv_{1} + \vu_{\alpha}, t + \frac{1}{\mu_{1}}\log(\frac{1}{\alpha}))$ and $\phi(\alpha' \vv_{1} + \vu_{\alpha'}, t + \frac{1}{\mu_{1}}\log(\frac{1}{\alpha'}))$, where $\alpha'$ is an arbitrary number smaller than $\alpha$. Let $\epsilon$ be an indeterminate number. Define $t_{1} = T_{\alpha}(\epsilon)$, $t_{2} = T_{\alpha'}(\epsilon)$, and $t_{0} = \frac{1}{\mu_{1}}\log(\frac{\alpha}{\alpha'})$. Then 
\[
t_{2} - t_{1} = \frac{1}{\mu_{1}}\log(\frac{\alpha}{\alpha'}) - \frac{1}{\mu_{1}}\log(\frac{1 + \kappa \alpha}{1 + \kappa \alpha'}) < t_{0},
\]
implying that $t_{2} - t_{0} < t_{1}$.

Applying Lemma E.3 with $r = \epsilon$, $t = t_{2} - t_{0}$, $\alpha = \alpha$, we obtain
\[
\phi(\alpha \vv_{1} + \vu_{\alpha}, t_{2} - t_{0}) = e^{(t_{2} - t_{0})\vJ(0)}(\alpha \vv_{1} + \vu_{\alpha}) + O(\epsilon^2).
\]
Similarly, applying Lemma E.3 with $r = \epsilon$, $t = t_{2}$, $\alpha = \alpha'$, we get
\[
\phi(\alpha' \vv_{1} + \vu_{\alpha'}, t_{2}) = e^{t_{2}\vJ(0)}(\alpha' \vv_{1} + \vu_{\alpha'}) + O(\epsilon^2).
\]
Since
\[
e^{(t_{2} - t_{0})\vJ(0)}\alpha \vv_{1} = e^{(t_{2} - t_{0})\mu_{1}}\alpha \vv_{1} = e^{t_{2}\mu_{1}}\alpha'\vv_{1},
\]
we have
\[
\|\phi(\alpha \vv_{1} + \vu_{\alpha}, t_{2} - t_{0}) - \phi(\alpha' \vv_{1} + \vu_{\alpha'}, t_{2})\|_{2} = \|e^{t_{2}\vJ(0)}\vu_{\alpha'} + e^{(t_{2} - t_{0})\vJ(0)}\vu_{\alpha}\|_{2} + O(\epsilon^2).
\]
Furthermore, we have
\[
\|e^{t_{2}\vJ(0)}\vu_{\alpha'}\|_{2} \leq e^{\mu_{2}t_{2}}c\alpha' = \left(\frac{\epsilon (1 + \kappa \alpha')}{\alpha'(1 + \kappa \epsilon)} \right)^{\frac{\mu_{2}}{\mu_{1}}}c\alpha',
\]
\[
\|e^{(t_{2} - t_{0})\vJ(0)}\vu_{\alpha}\|_{2} \leq e^{\mu_{2}(t_{2} - t_{0})}c\alpha = \left(\frac{\epsilon (1 + \kappa \alpha')}{\alpha(1 + \kappa \epsilon)} \right)^{\frac{\mu_{2}}{\mu_{1}}}c\alpha.
\]
Thus, we obtain
\begin{align*}
\|\phi(\alpha' \vv_{1} + \vu_{\alpha'}, t_{2}) - \phi(\alpha \vv_{1} + \vu_{\alpha}, t_{2} - t_{0})\|_{2} 
&\leq O(\epsilon^2) + \left(\frac{\epsilon (1 + \kappa \alpha')}{\alpha(1 + \kappa \epsilon)} \right)^{\frac{\mu_{2}}{\mu_{1}}}c\alpha + \left(\frac{\epsilon (1 + \kappa \alpha')}{\alpha'(1 + \kappa \epsilon)} \right)^{\frac{\mu_{2}}{\mu_{1}}}c\alpha'\\
&= O(\epsilon^2) + O\left((\frac{\epsilon}{\alpha})^{\frac{\mu_{2}}{\mu_{1}}}\right)\alpha.
\end{align*}
Applying Lemma E.3, we conclude that $\|\phi(\alpha' \vv_{1} + \vu_{\alpha'}, t_{2})\|_{2} \leq \epsilon$ and $\|\phi(\alpha \vv_{1} + \vu_{\alpha}, t_{2} - t_{0})\|_{2} \leq \epsilon$. Define $\Delta t = t + \frac{1}{\mu_{1}}\log(\frac{1}{\alpha'}) - t_{2} = t + \frac{1}{\mu_{1}}\log(\frac{1}{\epsilon}) + \frac{1}{\mu_{1}}\log(\frac{1 + \kappa \epsilon}{1 + \kappa \alpha'})$.

 Because $\Delta t \leq T_{\epsilon}(r) \iff \mu_{1}t-\log(1+\kappa \alpha^{\prime})\leq \log(\frac{r}{1+\kappa r})$, so we only need $\mu_1 t \le \log(\frac{r}{1+\kappa r})$ to satisfy $\Delta t \leq T_{\epsilon}(r)$.
 
\textbf{Case1:} Assume $\mu_1t \leq \log(\frac{R}{1+\kappa R})$. 

let $r=\frac{1}{e^{-\mu_1 t}-\kappa}$, then $r\leq R$ and $ \Delta t \leq T_{\epsilon}(r)$.
Applying Lemma E.5 with $\vtheta_{0} = \phi(\alpha' \vv_{1} + \vu_{\alpha'}, t_{2})$, $\hat{\vtheta}_{0} = \phi(\alpha \vv_{1} + \vu_{\alpha}, t_{2} - t_{0})$, $t = \Delta t$, $\alpha = \epsilon$, and $r=\frac{1}{e^{-\mu_1 t}-\kappa}$, we get
\begin{align*}
&\|\phi(\alpha \vv_{1} + \vu_{\alpha}, t + \frac{1}{\mu_{1}}\log(\frac{1}{\alpha})) - \phi(\alpha' \vv_{1} + \vu_{\alpha'}, t + \frac{1}{\mu_{1}}\log(\frac{1}{\alpha'}))\|_{2}\\
&\leq e^{\mu_{1} \Delta t + kr} \times \|\phi(\alpha' \vv_{1} + \vu_{\alpha'}, t_{2}) - \phi(\alpha \vv_{1} + \vu_{\alpha}, t_{2} - t_{0})\|_{2}\\
&\leq e^{\mu_{1}t + kr} \times O\left(\frac{1}{\epsilon}\right) \times \left(O(\epsilon^2) + O\left((\frac{\epsilon}{\alpha})^{\frac{\mu_{2}}{\mu_{1}}}\right)\alpha \right)\\
&= e^{\mu_{1}t + kr} \times \left(O(\epsilon) + O\left((\frac{\alpha}{\epsilon})^{1 - \frac{\mu_{2}}{\mu_{1}}}\right)\right).
\end{align*}
Setting $\epsilon = \alpha^s$ where $0 < s < 1$, we find
\[
e^{\mu_{1}t + kr} \times \left(O(\epsilon) + O\left((\frac{\alpha}{\epsilon})^{1 - \frac{\mu_{2}}{\mu_{1}}}\right)\right) = e^{\mu_{1}t + kr} \times O\left(\alpha^{\min\{s, (1 - s)(1 - \frac{\mu_{2}}{\mu_{1}})\}}\right).
\]
Choosing $s = \frac{\mu_{1} - \mu_{2}}{2\mu_{1} - \mu_{2}}$, we obtain the tightest bound:
\[
\|\phi(\alpha \vv_{1} + \vu_{\alpha}, t + \frac{1}{\mu_{1}}\log(\frac{1}{\alpha})) - \phi(\alpha' \vv_{1} + \vu_{\alpha'}, t + \frac{1}{\mu_{1}}\log(\frac{1}{\alpha'}))\|_{2} \leq e^{\mu_{1}t + kr} \times O\left(\alpha^{\frac{\mu_{1} - \mu_{2}}{2\mu_{1} - \mu_{2}}}\right).
\]

\textbf{Case2:} Assume $\mu_1t > \log(\frac{R}{1+\kappa R})$. 

Denote $t_s=\frac{1}{\mu_1}\log(\frac{R}{1+\kappa R})$ and $\tau:=t-t_s$. 
Then $\mu_1t_s \leq \log(\frac{R}{1+\kappa R})$. Applying results of \textbf{Case1}, we get:
\[
\|\phi(\alpha \vv_{1} + \vu_{\alpha}, t_s + \frac{1}{\mu_{1}}\log(\frac{1}{\alpha})) - \phi(\alpha' \vv_{1} + \vu_{\alpha'}, t_s+ \frac{1}{\mu_{1}}\log(\frac{1}{\alpha'}))\|_{2} \leq e^{\mu_{1}t_s + kR} \times O\left(\alpha^{\frac{\mu_{1} - \mu_{2}}{2\mu_{1} - \mu_{2}}}\right)
\]

Because $\phi(\vtheta,t)$ is locally Lipschitz over $\vtheta$, so
\begin{align}
   & \|\phi(\alpha \vv_{1} + \vu_{\alpha}, t + \frac{1}{\mu_{1}}\log(\frac{1}{\alpha})) - \phi(\alpha' \vv_{1} + \vu_{\alpha'}, t+ \frac{1}{\mu_{1}}\log(\frac{1}{\alpha'}))\|_{2}\\
   &= 
   \|\phi\left(\phi(\alpha \vv_{1} + \vu_{\alpha}, t_s + \frac{1}{\mu_{1}}\log(\frac{1}{\alpha})),\tau\right)-\phi\left(\phi(\alpha' \vv_{1} + \vu_{\alpha'}, t_s+ \frac{1}{\mu_{1}}\log(\frac{1}{\alpha'})),\tau\right)\|_{2}\\
   &=O(\|\phi(\alpha \vv_{1} + \vu_{\alpha}, t_s + \frac{1}{\mu_{1}}\log(\frac{1}{\alpha})) - \phi(\alpha' \vv_{1} + \vu_{\alpha'}, t_s+ \frac{1}{\mu_{1}}\log(\frac{1}{\alpha'}))\|_{2})\\
   &=O\left(\alpha^{\frac{\mu_{1} - \mu_{2}}{2\mu_{1} - \mu_{2}}}\right)
\end{align}

In both case,
$\phi(\alpha \vv_{1} + \vu_{\alpha}, t + \frac{1}{\mu_{1}}\log(\frac{1}{\alpha}))$ as a sequence of $\alpha$ satisfies the Cauchy criterion. Therefore, $\lim_{\alpha \to 0}\phi(\alpha \vv_{1} + \vu_{\alpha}, t + \frac{1}{\mu_{1}}\log(\frac{1}{\alpha})) := h(\vv_{1}, t)$ exists. Moreover, since $e^{\mu_{1}t + kr}$ is bounded in a neighborhood of $\vv_{1}$ and $t$, the convergence is uniform, ensuring that $h(\vv_{1}, t)$ is continuous with respect to $\vv_{1}$ and $t$.
\end{proof}


\subsection{Corollary of Theorem\ref{theorem1}}

\begin{corollary}
Assume the assumptions of Theorem~\ref{theorem3} hold. Let $\vv$ be a vector in the parameter space, and let $\vv_1$ be the projection of $\vv$ into the eigenspace corresponding to the largest eigenvalue of $\nabla g(\vtheta)|_{\vtheta=\mathbf{0}}$. Then the limit $h(\vv,t) := \lim_{\alpha \to 0} \phi(\alpha \vv, t + \frac{1}{\mu_{1}}\log\frac{1}{\alpha})$ exists, and $h(\vv,t) = h(\vv_1,t)$.
\label{Cor1}
\end{corollary}

\begin{proof}
We have $\phi(\alpha \vv, t + \frac{1}{\mu_{1}}\log\frac{1}{\alpha}) = \phi(\alpha \vv_1 + \alpha(\vv - \vv_1), t + \frac{1}{\mu_{1}}\log\frac{1}{\alpha})$.

Let $\vu_{\alpha} = \alpha(\vv - \vv_1)$. Then $\vu_{\alpha}$ satisfies:
\begin{itemize}
    \item[(i)] $\vu_{\alpha}$ is orthogonal to the eigenspace of $\mu_{1}$.
    \item[(ii)] There exists $c > 0$ such that for all $\alpha > 0$, $\|\vu_{\alpha}\|_{2} \leq c\alpha$.
\end{itemize}
Applying Theorem~\ref{theorem3}, we get $\lim_{\alpha \to 0} \phi(\alpha \vv, t + \frac{1}{\mu_{1}}\log\frac{1}{\alpha}) = \lim_{\alpha \to 0} \phi(\alpha \vv_1, t + \frac{1}{\mu_{1}}\log\frac{1}{\alpha})$, so $h(\vv,t) = h(\vv_1,t)$.
\end{proof}

\begin{corollary}
Assume the assumptions of Theorem~\ref{theorem3} hold. Let $h(\vv,t) := \lim_{\alpha \to 0} \phi(\alpha \vv, t + \frac{1}{\mu_{1}}\log\frac{1}{\alpha})$. Then for all $s > 0$, $h(s\vv,t) = h(\vv, t + \frac{1}{\mu_1}\log(s))$.
\label{Cor2}
\end{corollary}

\begin{proof}
According to the definition, we have $h(s\vv,t) = \lim_{\alpha \to 0} \phi(\alpha s\vv, t + \frac{1}{\mu_{1}}\log\frac{1}{\alpha})$. Let $\alpha' = \alpha s$, then $h(s\vv,t) = \lim_{\alpha' \to 0} \phi(\alpha' \vv, t + \frac{1}{\mu_1}\log(s) + \frac{1}{\mu_{1}}\log\frac{1}{\alpha'}) = h(\vv, t + \frac{1}{\mu_1}\log(s))$.
\end{proof}

\begin{corollary}
Assume the assumptions of Theorem~\ref{theorem3} hold. Let $h(\vv,t) := \lim_{\alpha \to 0} \phi(\alpha \vv, t + \frac{1}{\mu_{1}}\log\frac{1}{\alpha})$.Let $T_{\vv} := \{h(\vv,t) : t \in \mathbb{R}\}$. Let $\vv_1$ be the projection of $\vv$ into the eigenspace of the largest eigenvalue of $\nabla g(\vtheta)|_{\vtheta=\mathbf{0}}$. Then $T_{\vv} = T_{\vv_1}$, and for all $c > 0$, $T_{c\vv} = T_{\vv}$. If $\vv \neq \mathbf{0}$, then $T_{\vv} = T_{\frac{\vv_1}{\|\vv_1\|_2}}$.
\label{Cor3}
\end{corollary}

\begin{proof}
From Corollary~\ref{Cor1}, we have $h(\vv,t) = h(\vv_1,t)$, so $T_{\vv} = T_{\vv_1}$. From Corollary~\ref{Cor2}, we have for all $c > 0$, $h(c\vv,t) = h(\vv, t + \frac{1}{\mu_1}\log(c))$, so $T_{c\vv} = T_{\vv}$ for all $c > 0$.
\end{proof}

\begin{remark}
Corollary~\ref{Cor3} implies that the trajectory of parameters is determined by $\frac{\vv_1}{\|\vv_1\|_2}$.
\end{remark}

\begin{proof}[\textbf{Proof of Theorem~\ref{theorem1}:}]
    Denote $f_i(\vtheta)=f_{\vtheta}(\vx_i)$. Then $\ell(\vtheta)=\frac{1}{2}\sum_{i=1}^n(f_i(\vtheta)-y_i)^2$. The gradient of $\ell$ at $\vtheta$ is given by $-\nabla \ell(\vtheta)=\sum_{i=1}^n (y_i-f_i(\vtheta))\nabla f_i(\vtheta)$. Because $\nabla f_i(\mathbf{0})=\mathbf{0}$, we have $-\nabla \ell(\mathbf{0})=\mathbf{0}$.

    The Hessian of $-\ell$ at $\mathbf{0}$ is $-\nabla^2 \ell(\mathbf{0})=\sum_{i=1}^n y_i \nabla^2 f_i(\mathbf{0})$, which is a block diagonal matrix with blocks $D_i$ for $i=1,\ldots,m$.
    
    $$-\nabla^2 \ell(\mathbf{0})=
 \left(\begin{array}{cccc}
  D_1 & 0 & \cdots & 0 \\
  0 & D_2 & \cdots & 0 \\
  \vdots & \vdots & \ddots & \vdots \\
  0 & 0 & \cdots & D_m
  \end{array}\right)
  $$
    
    where $D_i$ is given by
    \[
    D_i=
    \begin{pmatrix}
    0 & \vgamma^\top\\
    \vgamma & \mathbf{0}_{d \times d}   \\
    \end{pmatrix},
    i=1,2,\ldots ,m\]

    The maximum eigenvalue of $D_i$ is $\|\vgamma\|_2$, and the corresponding eigenvector is $(\|\vgamma\|_2,\vgamma)$. Thus, the maximum eigenvalue of $-\nabla^2 \ell(\mathbf{0})$ is $\|\vgamma\|_2$. Denote the eigenspace of $\|\vgamma\|_2$ by $\vV$, then the projection of $\vtheta^*$ onto $\vV$ is determined by $\vC$. It is easy to verify that the assumptions of Theorem~\ref{theorem3} hold, by defining $g(\vtheta):=-\nabla \ell(\vtheta)$. Then, by Corollary~\ref{Cor1}, $h(\vtheta,t)=\lim_{\alpha \to 0} \phi(\alpha \vtheta,t+\frac{1}{\mu}\log\frac{1}{\alpha})$ exists, and $h(\vtheta,t)$ as a function of $\vtheta$ is determined by $\vC$. By Corollary~\ref{Cor3}, $T_{\vtheta}$ is  determined by $\frac{\vC}{\|\vC\|_2}$.
\end{proof}

\begin{proposition}
    Assume the assumptions of Theorem~\ref{theorem3} hold. Let $h(\vv_0, t) := \lim_{\alpha \to 0} \phi(\alpha \vv_0, t + \frac{1}{\mu_1}\log\frac{1}{\alpha})$. Then $h(\vv_0, t)$ as a function of $t$ is differentiable, and $\frac{d}{dt}h(\vv_0, t) = g(h(\vv_0, t))$.
\label{proposition1}
\end{proposition}

\begin{proof}
    Since $\frac{d}{dt}\phi(\alpha \vv_0, t + \frac{1}{\mu}\log\frac{1}{\alpha}) = g\left(\phi(\alpha \vv_0, t + \frac{1}{\mu}\log\frac{1}{\alpha})\right)$, we have
    \begin{equation}
        \phi(\alpha \vv_0, t + \frac{1}{\mu}\log\frac{1}{\alpha}) = h(\vv_0, 0) + \int_{0}^{t} g\left(\phi(\alpha \vv_0, s + \frac{1}{\mu}\log\frac{1}{\alpha})\right) ds.
    \label{eq:derivative}
    \end{equation}

Consider $\phi(\alpha \vv_0, t + \frac{1}{\mu}\log\frac{1}{\alpha})$ as a function of $t$ and $\alpha$, and extend the value of $\phi(\alpha \vv_0, t + \frac{1}{\mu}\log\frac{1}{\alpha})$ at $\alpha = 0$ by $\lim_{\alpha \to 0} \phi(\alpha \vv_0, t + \frac{1}{\mu}\log\frac{1}{\alpha})$. Then $\phi$ is continuous for $(t, \alpha) \in D = [0, t_0] \times [0, 1]$. Thus, there exists $\delta > 0$ such that $\phi(D) \subset B_{\delta}(\mathbf{0})$. Because $g(\vtheta)$ is continuously differentiable, it is Lipschitz continuous in $B_{\delta}(\mathbf{0})$.

In Equation~\ref{eq:derivative}, let $\alpha \to 0$. Since $\phi(\alpha \vv_0, s + \frac{1}{\mu}\log\frac{1}{\alpha})$ uniformly converges to $h(\vv_0, s)$ and $g$ is Lipschitz continuous, we obtain
\[
h(\vv_0, t) = h(\vv_0, 0) + \int_{0}^{t} g(h(\vv_0, s)) ds.
\]
Therefore, $h(\vv_0, t)$ is differentiable over $t$, and $\frac{d}{dt}h(\vv_0, t) = g(h(\vv_0, t))$.
\end{proof}

\subsection{Proof of Theorem\ref{theorem2}}
\label{Proof of Theorem2}

Next we begin by establishing Theorem~\ref{theorem4}, and then leverage it to validate Theorem~\ref{theorem2}.

Consider the gradient flow of the loss function $\ell$:
\begin{equation}
\begin{aligned}
    &\frac{d\vtheta}{dt} = -\nabla \ell(\vtheta(t)), \\
    &\vtheta(0) = \vtheta_{0},
\end{aligned}
\label{eq:gd}
\end{equation}
and denote the solution of \ref{eq:gd} as $\phi(\vtheta_{0},t)$.

\begin{theorem}
Assume the following conditions:
\begin{itemize}
    \item[(i)] $\ell(\vtheta)$ is an analytic and nonnegative function.
    \item[(ii)] $\mathbf{0}$ is a strict saddle point of $\ell(\vtheta)$.
\end{itemize}
Denote $\vV_1$ to be the eigenspace corresponding to the largest eigenvalue of $-\nabla^2 \ell (\vtheta)|_{\vtheta=\mathbf{0}}$. Let $\vv_1$ be a vector in $\vV_1$. The vector $\vu_{\alpha}$ is arbitrary, subject to the conditions:
\begin{itemize}
    \item[(i)] $\vu_{\alpha}$ is orthogonal to $\vV_1$.
    \item[(ii)] There exists $c > 0$ such that for all $\alpha > 0$, $\|\vu_{\alpha}\|_{2} \leq c\alpha$.
\end{itemize}
Denote the solution of \ref{eq:gd} as $\phi(\vtheta_{0},t)$. Then $h(\vv,t) = \lim_{\alpha \to 0} \phi(\alpha \vv + \vu_{\alpha}, t + \frac{1}{\mu_{1}}\log\frac{1}{\alpha})$ exists.
Given $\vv_0$, if $h(\vv_0,t)$ is bounded for all $t \geq 0$, then the limit $\lim_{t \to \infty} h(\vv_0,t)$ exists. Furthermore, if this limit is not a saddle point of $\ell(\vtheta)$, then there exists a neighborhood $B_{\delta}(\vv_0)$ of $\vv_0$, and an $\alpha_0 > 0$, such that the limit $\lim_{t \to \infty} \phi(\alpha \vv + \vu_{\alpha}, t + \frac{1}{\mu_{1}}\log\frac{1}{\alpha})$ exists for all $\vv \in B_{\delta}(\vv_0)$ and $0 < \alpha < \alpha_0$, and
\[
\lim_{t \to \infty} h(\vv_0,t) = \lim_{\alpha \to 0} \lim_{t \to \infty} \phi(\alpha \vv_0 + \vu_{\alpha}, t + \frac{1}{\mu_{1}}\log\frac{1}{\alpha}).
\]
Moreover, the limit $\lim_{t \to \infty} h(\vv,t)$ is a continuous function of $\vv$ at $\vv_0$.
\label{theorem4}
\end{theorem}

\begin{proof}
We attempt to apply Theorem~\ref{theorem3} by setting $g(\vtheta) = -\nabla \ell(\vtheta)$ and checking the conditions for $g(\vtheta)$:
\begin{itemize}
    \item Since $\ell(\vtheta)$ is analytic, $g(\vtheta)$ is twice differentiable.
    \item As $\mathbf{0}$ is a strict saddle point of $\ell(\vtheta)$, we have $g(\mathbf{0}) = \mathbf{0}$.
    \item The matrix $\nabla g(\vtheta)|_{\vtheta=\mathbf{0}} = -\nabla^2 \ell (\vtheta)|_{\vtheta=\mathbf{0}}$ is symmetric and thus diagonalizable.
    \item Because $\mathbf{0}$ is a strict saddle point, the largest eigenvalue of $\nabla g(\vtheta)|_{\vtheta=\mathbf{0}}$ is positive.
\end{itemize}
Therefore, $g(\vtheta)$ satisfies the conditions of Theorem~\ref{theorem3}. Applying Theorem~\ref{theorem3}, we conclude that $h(\vv,t) = \lim_{\alpha \rightarrow 0} \phi(\alpha \vv + \vu_{\alpha}, t + \frac{1}{\mu_{1}}\log\frac{1}{\alpha})$ exists.

We then proceed in two steps to prove Theorem~\ref{theorem4}.

(i) First, we prove the existence of $\lim_{t \to \infty} h(\vv_0,t)$ under the condition that $h(\vv_0,t)$ is bounded for $t \geq 0$.

Let $\vtheta(t) = h(\vv_0,t)$. By Proposition \ref{proposition1}, $\vtheta'(t)=-\nabla \ell(\vtheta(t))$. Because $\vtheta(t)$ is bounded for all $t \geq 0$, there exists a sequence $\{t_n\}$ such that $\lim_{n \to \infty} t_n = +\infty$, and $\hat{\vtheta} = \lim_{n \to \infty} \vtheta(t_n)$ exists. Since
\[
\frac{d}{dt} \ell(\vtheta(t)) = \langle -\nabla \ell(\vtheta(t)), \vtheta'(t) \rangle = -\|\nabla \ell(\vtheta(t))\|_2^2,
\]
$\ell(\vtheta(t))$ is monotonically decreasing over $t$, and $\ell(\vtheta(t)) \geq \ell(\hat{\vtheta})$ for all $t \geq 0$. Furthermore, $\int_0^{+\infty} \|\nabla \ell(\vtheta(t))\|_2^2 \, dt \leq \ell(\vtheta(0))$. Therefore, $\lim_{t \to \infty} \|\nabla \ell(\vtheta(t))\|_2 = 0$, and $\nabla \ell(\hat{\vtheta}) = \lim_{t \to \infty} \nabla \ell(\vtheta(t_n)) = \mathbf{0}$.

By Łojasiewicz's inequality\cite{lojasiewicz1965ensembles}, there exists $C > 0$ and $0 < \mu < 1$, and a neighborhood $B_{\delta}(\hat{\vtheta})$ such that
\[
\|\nabla \ell(\vtheta)\|_2 \geq C |\ell(\vtheta) - \ell(\hat{\vtheta})|^\mu, \quad \forall \vtheta \in B_{\delta}(\hat{\vtheta}).
\]
Since $\ell(\vtheta(t)) \geq \ell(\hat{\vtheta})$, it follows that
\[
\|\nabla \ell(\vtheta(t))\|_2 \geq C (\ell(\vtheta(t)) - \ell(\hat{\vtheta}))^\mu, \quad \forall \vtheta \in B_{\delta}(\hat{\vtheta}).
\]
Given that $\lim_{t \to \infty} \vtheta(t_n) = \hat{\vtheta}$, there exists an $n$ such that $\|\vtheta(t_n) - \hat{\vtheta}\|_2 < \frac{\delta}{2}$ and $|\ell(\vtheta(t_n)) - \ell(\hat{\vtheta})| < \frac{1}{2} C (1 - \mu) \delta^{\frac{1}{1 - \mu}}$. Because
\[
\frac{d}{dt} \ell(\vtheta(t)) = -\|\nabla \ell(\vtheta(t))\|_2^2 = -\|\nabla \ell(\vtheta(t))\|_2 \times \left\|\frac{d\vtheta}{dt}\right\|_2 \leq -C (\ell(\vtheta(t)) - \ell(\hat{\vtheta}))^\mu \left\|\frac{d\vtheta}{dt}\right\|_2,
\]
we have
\[
\frac{d}{dt} (\ell(\vtheta(t)) - \ell(\hat{\vtheta}))^{1 - \mu} \leq -(1 - \mu) C \left\|\frac{d\vtheta}{dt}\right\|_2.
\]
Thus,
\[
\int_{t_n}^{t} \left\|\frac{d\vtheta}{dt}\right\|_2 \, dt \leq \frac{1}{C(1 - \mu)} (\ell(\vtheta(t_n)) - \ell(\hat{\vtheta}))^{1 - \mu}.
\]
Since $\|\vtheta(t) - \vtheta(t_n)\|_2 \leq \int_{t_n}^{t} \left\|\frac{d\vtheta}{dt}\right\|_2 \, dt$, it follows that
\[
\|\vtheta(t) - \vtheta(t_n)\|_2 \leq \frac{1}{C(1 - \mu)} (\ell(\vtheta(t_n)) - \ell(\hat{\vtheta}))^{1 - \mu} < \frac{\delta}{2}.
\]
Therefore,
\[
\|\vtheta(t) - \hat{\vtheta}\|_2 \leq \|\vtheta(t) - \vtheta(t_n)\|_2 + \|\vtheta(t_n) - \hat{\vtheta}\|_2 < \delta,
\]
so for all $t \geq t_n$, $\vtheta(t) \in B_{\delta}(\hat{\vtheta})$. Thus, we can apply Łojasiewicz's inequality for all $t \geq t_n$. Consequently,
\[
\int_{t_n}^{t} \left\|\frac{d\vtheta}{dt}\right\|_2 \, dt \leq \frac{1}{C(1 - \mu)} (\ell(\vtheta(t_n)) - \ell(\hat{\vtheta}))^{1 - \mu}, \quad \forall t \geq t_n.
\]
Thus, the length of the trajectory of $\vtheta(t)$ is finite, which implies that $\lim_{t \to \infty} \vtheta(t)$ exists.

\begin{lemma}
    Denote the solution of Equation~\ref{eq:gd} as $\phi(\vtheta_{0}, t)$. Assume the following:
    \begin{enumerate}[(i)]
        \item $\ell(\vtheta)$ is an analytic and nonnegative function.
        \item $\overline{\vtheta} = \lim_{t \to \infty} \phi(\vv_0, t)$ exists.
        \item $\overline{\vtheta}$ is a local minimum of $\ell(\vtheta)$.
    \end{enumerate}
    Then there exists a neighborhood of $\vv_0$, denoted $B_{\delta}(\vv_0)$, such that for all $\vv \in B_{\delta}(\vv_0)$, the limit $\lim_{t \to \infty} \phi(\vv, t)$ exists and is continuous at $\vv_0$.
\label{Lemma1}
\end{lemma}

\begin{proof}
By Łojasiewicz's inequality, there exists $C > 0$ and $0 < \mu < 1$, and a neighborhood $B_{\epsilon_0}(\overline{\vtheta})$ such that
\[
\|\nabla \ell(\vtheta)\|_2 \geq C |\ell(\vtheta) - \ell(\overline{\vtheta})|^\mu, \quad \forall \vtheta \in B_{\epsilon_0}(\overline{\vtheta}).
\]

Since $\overline{\vtheta} = \lim_{t \to \infty} \phi(\vv_0, t)$, for all $\epsilon \in (0, \epsilon_0)$, there exists $t_0 > 0$ such that $\|\phi(\vv_0, t_0) - \overline{\vtheta}\|_2 < \frac{\epsilon}{4}$ and $|\ell(\phi(\vv_0, t_0)) - \ell(\overline{\vtheta})| < \frac{1}{4} C (1 - \mu) \epsilon^{\frac{1}{1 - \mu}}$.

Because $\phi(\vtheta, t_0)$ is locally Lipschitz continuous over $\vtheta$, there exists $L > 0$ and $\delta > 0$ such that $\|\phi(\vv, t_0) - \phi(\vv_0, t_0)\|_2 < \frac{\epsilon}{4}$ for all $\vv \in B_{\delta}(\vv_0)$ and $|\ell(\phi(\vv, t_0)) - \ell(\phi(\vv_0, t_0))| < \frac{1}{4} C (1 - \mu) \epsilon^{\frac{1}{1 - \mu}}$ for all $\vv \in B_{\delta}(\vv_0)$.

Thus, for all $\vv \in B_{\delta}(\vv_0)$, we have $\|\phi(\vv, t_0) - \overline{\vtheta}\|_2 < \|\phi(\vv, t_0) - \phi(\vv_0, t_0)\|_2 + \|\phi(\vv_0, t_0) - \overline{\vtheta}\|_2 < \frac{\epsilon}{2}$ and $|\ell(\phi(\vv, t_0)) - \ell(\overline{\vtheta})| \leq |\ell(\phi(\vv, t_0)) - \ell(\phi(\vv_0, t_0))| + |\ell(\phi(\vv_0, t_0)) - \ell(\overline{\vtheta})| < \frac{1}{2} C (1 - \mu) \epsilon^{\frac{1}{1 - \mu}}$ for all $\vv \in B_{\delta}(\vv_0)$.

Applying Łojasiewicz's inequality as in Theorem~\ref{theorem3}, we get
\begin{equation}
\int_{t_0}^{t} \left\|\frac{\partial \phi(\vv, \tau)}{\partial \tau}\right\|_2 \, d\tau \leq \frac{1}{C(1 - \mu)} (\ell(\phi(\vv, t_0)) - \ell(\overline{\vtheta}))^{1 - \mu}.
\label{eq:lemma1}
\end{equation}
and
\begin{equation}
\|\phi(\vv, t) - \phi(\vv, t_0)\|_2 \leq \frac{1}{C(1-\mu)} (\ell(\phi(\vv, t_0)) - \ell(\overline{\vtheta}))^{1-\mu} < \frac{\epsilon}{2}.
\label{eq:lemma2}
\end{equation}

Therefore, $\|\phi(\vv, t) - \overline{\vtheta}\|_2 \leq \|\phi(\vv, t) - \phi(\vv, t_0)\|_2 + \|\phi(\vv, t_0) - \phi(\vv_0, t_0)\|_2 + \|\phi(\vv_0, t_0) - \overline{\vtheta}\|_2 < \frac{\epsilon}{2} + \frac{\epsilon}{4} + \frac{\epsilon}{4} = \epsilon$.

From Equation~\ref{eq:lemma1}, we know that for all $\vv \in B_{\delta}(\vv_0)$, the length of the trajectory of $\phi(\vv, t)$ is finite, so $\lim_{t \to \infty} \phi(\vv, t)$ exists.

From Equation~\ref{eq:lemma2}, we know that for all $\epsilon \in (0, \epsilon_0)$, there exists $\delta > 0$ and $t_0 > 0$ such that for all $\vv \in B_{\delta}(\vv_0)$ and for all $t > t_0$, we have $\|\phi(\vv, t) - \overline{\vtheta}\|_2 < \epsilon$. Letting $t \to \infty$, we conclude: for all $\epsilon \in (0, \epsilon_0)$, there exists $\delta > 0$ such that for all $\vv \in B_{\delta}(\vv_0)$, $\|\lim_{t \to \infty} \phi(\vv, t) - \overline{\vtheta}\|_2 < \epsilon$. Therefore, $\lim_{t \to \infty} \phi(\vv, t)$ is continuous at $\vv_0$.
\end{proof}

Now we are back to proof of second part of Theorem~\ref{theorem4}. We assume that $\lim_{t\to \infty}h(\vv_0,t)$ exists and is a local minimum of $\ell(\vtheta)$. We want to prove the following:
\begin{enumerate}
    \item There exists a neighborhood $B_{\delta}(\vv_0)$ of $\vv_0$, and an $\alpha_0 > 0$, such that the limit $\lim_{t \to \infty} \phi(\alpha \vv + \vu_{\alpha}, t + \frac{1}{\mu_{1}}\log\frac{1}{\alpha})$ exists for all $\vv \in B_{\delta}(\vv_0)$ and $0 < \alpha < \alpha_0$.
    \item $\lim_{t \to \infty} h(\vv_0,t) = \lim_{\alpha \to 0} \lim_{t \to \infty} \phi(\alpha \vv_0 + \vu_{\alpha}, t + \frac{1}{\mu_{1}}\log\frac{1}{\alpha}).$
    \item The limit $\lim_{t \to \infty} h(\vv,t)$ is a continuous function of $\vv$ at $\vv_0$.
\end{enumerate}

Denote $q(\vv,\alpha) := \phi(\alpha \vv + \vu_{\alpha}, \frac{1}{\mu_{1}}\log\frac{1}{\alpha})$. By Theorem~\ref{theorem3}, $\lim_{\alpha \to 0}q(\vv,\alpha)$ exists and is continuous over $\vv$.

We have 
\[
\phi(\alpha \vv + \vu_{\alpha}, t + \frac{1}{\mu_{1}}\log\frac{1}{\alpha}) = \phi(\phi(\alpha \vv + \vu_{\alpha}, \frac{1}{\mu_{1}}\log\frac{1}{\alpha}), t) = \phi(q(\vv,\alpha), t).
\]
And 
\[
h(\vv,t) = \lim_{\alpha \to 0}\phi(\alpha \vv + \vu_{\alpha}, t + \frac{1}{\mu_{1}}\log\frac{1}{\alpha}) = \phi(\lim_{\alpha \to 0}q(\vv,\alpha), t).
\]

Denote $q(\vv,0) := \lim_{\alpha \to 0}q(\vv,\alpha)$. By setting $\vv_0$ in Lemma~\ref{Lemma1} equal to $q(\vv_0,0)$, we get: there exists a neighborhood of $q(\vv_0,0)$, denoted $B_{\delta}(q(\vv_0,0))$, such that for all $\vv \in B_{\delta}(q(\vv_0,0))$, the limit $\lim_{t \to \infty} \phi(\vv, t)$ exists and is continuous at $q(\vv_0,0)$.

Because $q(\vv,\alpha)$ is continuous at $(\vv_0,0)$, the pre-image $q^{-1}(B_{\delta}(q(\vv_0,0)))$ is open. Thus, there exists a neighborhood $B_{\delta}(\vv_0)$ of $\vv_0$, and an $\alpha_0 > 0$, such that for all $\vv \in B_{\delta}(\vv_0)$ and $0 < \alpha < \alpha_0$, the limit $\lim_{t \to \infty} \phi(\alpha \vv + \vu_{\alpha}, t + \frac{1}{\mu_{1}}\log\frac{1}{\alpha})$ exists.

Since $q(\vv,\alpha)$ is continuous at $(\vv_0,0)$, and $\lim_{t\to \infty}\phi(\vv,t)$ is continuous at $q(\vv_0,0)$, it follows that $\lim_{t\to \infty}\phi(q(\vv,\alpha),t)$ is continuous at $(\vv_0,0)$. Therefore, we have
\[
\lim_{\alpha\to 0}\lim_{t\to \infty}\phi(q(\vv_0,\alpha),t) = \lim_{t\to \infty}\phi(q(\vv_0,0),t) = \lim_{t\to \infty}h(\vv_0,t),
\]
and $\lim_{t\to \infty}h(\vv,t)$ is continuous at $\vv_0$. 
\end{proof}

\begin{proof}[\textbf{Proof of Theorem~\ref{theorem2}}]
Under the notations and assumptions of Theorem~\ref{theorem2}, we have
\[
-\nabla^2 \ell(\mathbf{0})=
\begin{pmatrix}
D_1 & 0 & \cdots & 0 \\
0 & D_2 & \cdots & 0 \\
\vdots & \vdots & \ddots & \vdots \\
0 & 0 & \cdots & D_m
\end{pmatrix},
\]
where each $D_i$ is defined as
\[
D_i=
\begin{pmatrix}
0 & \vgamma^\top \\
\vgamma & \mathbf{0}_{d \times d}
\end{pmatrix},
\]
for $i=1,2,\ldots,m$.

The maximum eigenvalue of $D_i$ is $\|\vgamma\|_2$. Therefore, the maximum eigenvalue of $-\nabla^2 \ell(\mathbf{0})$ is $\|\vgamma\|_2 > 0$, indicating that $\mathbf{0}$ is a strict saddle point of $\ell(\vtheta)$. Given that $\ell(\vtheta)$ is an analytic and nonnegative function, and $\mathbf{0}$ is a strict saddle point of $\ell(\vtheta)$, the conditions of Theorem~\ref{theorem4} are satisfied. Let us denote the eigenspace corresponding to $\|\vgamma\|_2$ by $\vV$. Then the projection of $\vtheta$ onto $\vV$ is determined by $\vC$. We now proceed to verify the assertions of Theorem~\ref{theorem2} point by point.

\begin{enumerate}[(i)]
    \item By applying Theorem~\ref{theorem4}, we confirm that the limit 
    \[
    h(\vtheta,t)=\lim_{\alpha \to 0}\phi(\alpha \vtheta, t+\frac{1}{\|\vgamma\|_2}\log\frac{1}{\alpha})
    \]
    exists.
    
    \item When the set $\{h(\vtheta,t) : t \geq 0\}$ is bounded, Theorem~\ref{theorem4} ensures that the limit $\lim_{t \to \infty}h(\vtheta,t)$ exists. Let $\hat{\vtheta}$ be the projection of $\vtheta$ onto $\vV$. Since $h(\vtheta,t)=h(\hat{\vtheta},t)$ and for any $\alpha > 0$, $h(\alpha \vtheta,t)=h(\vtheta,t+\frac{1}{\|\vgamma\|_2}\log \alpha)$, it follows that $h(\vtheta,t)=h\left(\frac{\hat{\vtheta}}{\|\hat{\vtheta}\|_2},t+\frac{1}{\|\vgamma\|_2}\log \frac{1}{\|\hat{\vtheta}\|_2}\right)$. Given the existence of $\lim_{t \to \infty}h(\vtheta,t)$, we have
    \[
    \lim_{t\to \infty}h(\vtheta,t)=\lim_{t\to \infty} h\left(\frac{\hat{\vtheta}}{\|\hat{\vtheta}\|_2},t\right).
    \]
    As $\frac{\hat{\vtheta}}{\|\hat{\vtheta}\|_2}$ is determined by $\frac{\vC}{\|\vC\|_2}$, the limit $\lim_{t\to \infty}h(\vtheta,t)$ is  determined by $\frac{\vC}{\|\vC\|_2}$.

    \item Applying Theorem~\ref{theorem4} again, we deduce that if $\lim_{t \to \infty} h(\vtheta_0,t)$ exists and is not a saddle point of $\ell(\vtheta)$, then
    \[
    \lim_{t \to \infty} h(\vtheta_0,t) = \lim_{\alpha \to 0} \lim_{t \to \infty} \phi\left(\alpha \vtheta_0, t + \frac{1}{\|\vgamma\|_2}\log\frac{1}{\alpha}\right).
    \]
    Furthermore, the limit $\lim_{t \to \infty} h(\vtheta,t)$ is continuous at $\vtheta_0$.
\end{enumerate}
\end{proof}

\section{Experiment of higher dimensional input}
\label{Appen:higher dimensional input}
Consider a neuron network with high dimensional input$f_{\vtheta}(\vx)=\sum_{i=1}^m a_i\sigma(\vw_i^{\top} \vx)$. Theorem~\ref{theorem1} shows that under small initialzation, the trajectory is determined by $\frac{\vC(\vtheta)}{\|\vC(\vtheta)\|_2}$. Because $\vC(\vtheta)$ is an $m$ dimensional vector, so the trajectory of $\vtheta$ under all initialization is restricted to a $m$ dimensional manifold. Note the dimension of manifold of $\vtheta$ is independent of dimension of input. So even when input is very high dimensional, the trace of parameters stays in a low dimensional.

In Figure~\ref{high dimensional loss and parameters}, we train a neural network for $m=2$ and $d=500$. The sample size is $500$. The initialization of parameters is a Gaussian distribution centered at $\mathbf{0}$, with its standard in figure's legend. By scaling the second neuron, we manage to make $C_1/C_2$ to be a constant in each subfigure. In this figure, the loss curve and parameters curve matches pefectly among different trials. The results show that, there existing a limiting trace of parameters as initialization scale approaches $0$. Moreover, the trace of parameters is  determined by $C_1/C_2$.

In Figure~\ref{high dimensional Q^*}, we conduct experiments for $m=2$ and $d=3$. $f_{\vtheta}(\vx)=a_1\tanh(\vw_1^{\top}\vx+b_1)+a_2\tanh(\vw_2^{\top}\vx+b_2)$, with $x \in \sR^3$. The target function is $f^*(\vx)=\tanh(\mathbf{1}^T\vx+1)$. The training data $(\vx_i,y_i)_{i=1}^n$ is obtained by $y_i=f^*(\vx_i)$ and $\{\vx_i\}_{i=1}^n$ draw independently from uniform distribution on interval $[-2,2]^d$. The sample size $n=10$. It is seen that in whatever initialization, the parameters all converge to $Q^*$, thus the generalization error is zero. Besides recovery, there is a clear sign that network initialized with $\Tilde{c}$ closer to $1$ will converge to $Q^1$ and network initialized with $\Tilde{c}$ closer to $0$ will converge to $Q^2$. This phenomenon is in accordance to what we observed in one-dimensional experiment.

In Figure~\ref{high dimensional, c and genloss}, we plot generalization error in high dimensional case. The network model is $f_{\vtheta}(\vx)=a_1\tanh(\vw_1^{\top}\vx+b_1)+a_2\tanh(\vw_2^{\top}\vx+b_2)$, with $x \in \sR^3$. The target function is $f^*(\vx)=\tanh(\mathbf{1}^T\vx+1)$. The setting of experiment is $m=2$ and $d=3$. In this case, optimistic sample size is $5$, the separation sample size of $Q^2$ is $6$, the separation sample size of $Q^1$ is $9$. It is seen that when the sample size is larger than optimistic sample size, the network can recover at $\Tilde{c}=1$ or $\Tilde{c}=0$. When sample size is above separation sample size of $Q^1$, all initialization can recover target function. At separation sample size of $Q^2$, which is $n=6$, there is a small probability of initialization cannot recover target function. When sample size is above $6,7,8$, the probability of recovery increases. 

In Figure~\ref{high dimension, converge of different n}, we show that why when $n=6,7,8$, recovery fails to happen at $\Tilde{c}=0$. Note that when sample size reaches the separation size of $Q^2$, it only guarantees that separation of $Q^2$ is almost every. That is to say, it allows that a subset with zero-measure with respect to $Q^2$ is not separated. Unfortunately, the origin of $Q^2$ belongs to the set. There exists other global minimum rather than $Q^2$ around origin of $Q^2$. Besides, $\Tilde{c}=0$ will lead to converge to origin of $Q^2$. So it is possible for $\Tilde{c}$ near $0$ to converge to these imperfect global minimum, thus failing in recovery. Overall, our understanding is that, both separation sample size and optimistic sample size is a lower bound for recovery. The recovery need not to happen at these sample size, but under these sample size, the recovery could not happen.

\begin{figure}[htbp]
\centering
\subfigure[$c=0.2$]{
\includegraphics[width=0.37\linewidth]{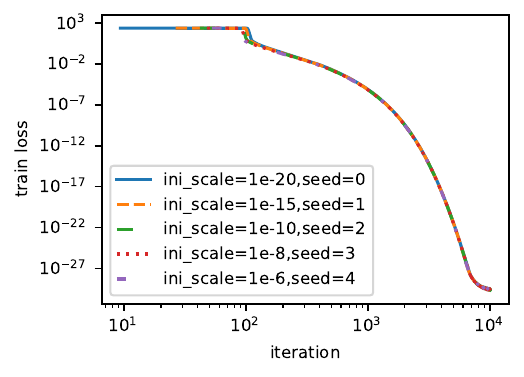}
\label{fig:c02_loss}}
\subfigure[$c=0.2$]{
\includegraphics[width=0.58\linewidth]{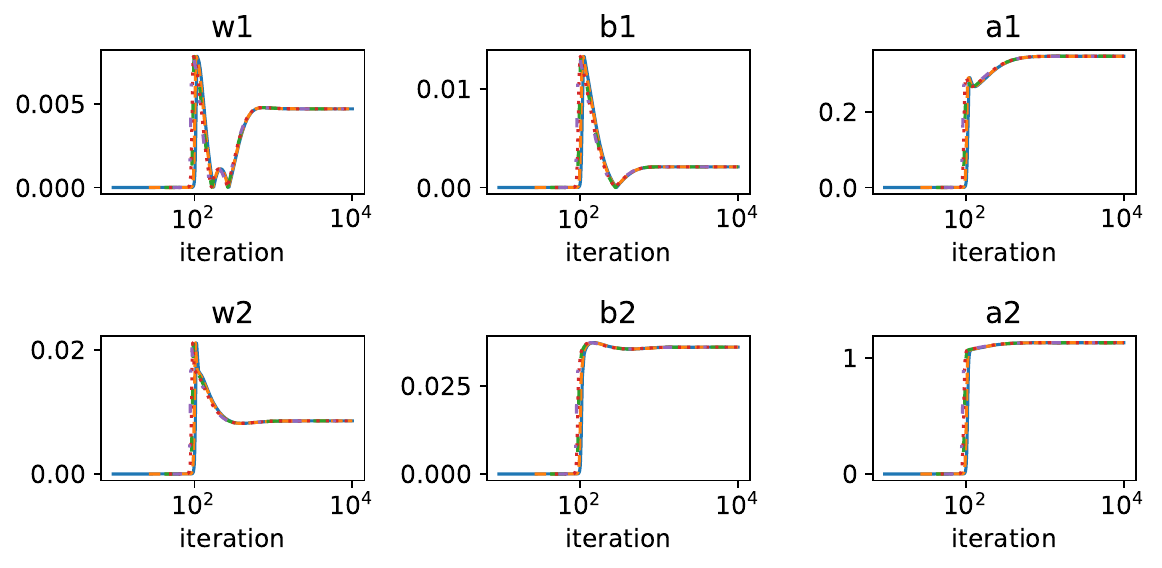}
\label{fig:c02_trajectory}}
\subfigure[$c=0.8$]{
\includegraphics[width=0.37\linewidth]{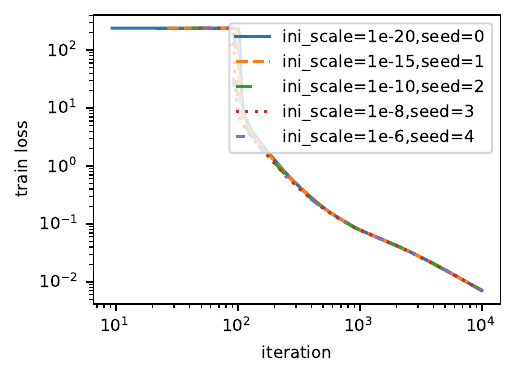}
\label{fig:c08_loss}}
\subfigure[$c=0.8$]{
\includegraphics[width=0.58\linewidth]{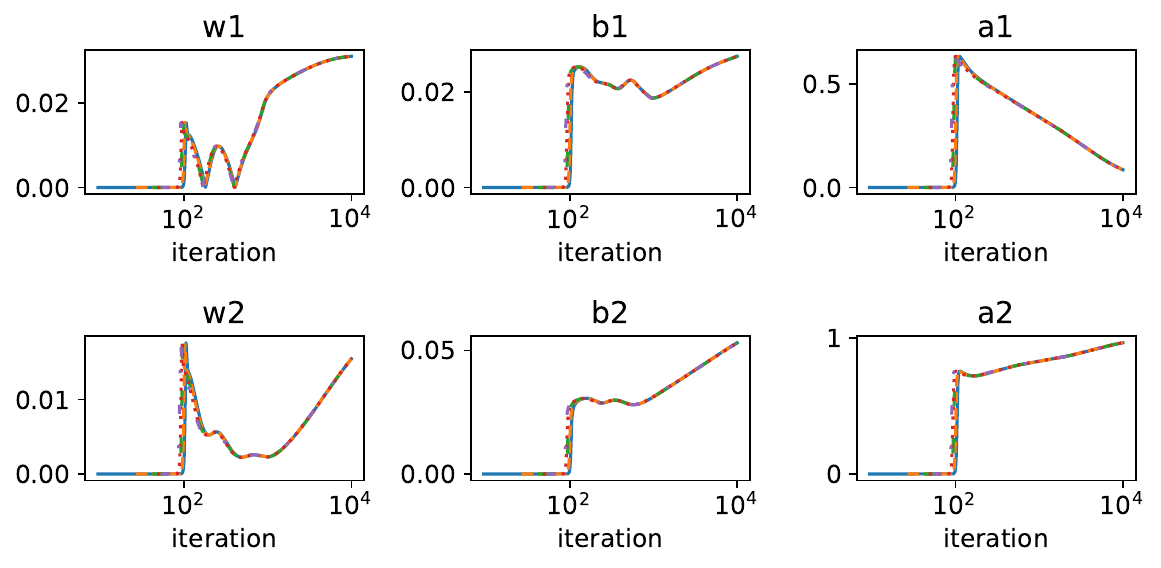}
\label{fig:c08_trajectory}}
\subfigure[$c=2$]{
\includegraphics[width=0.37\linewidth]{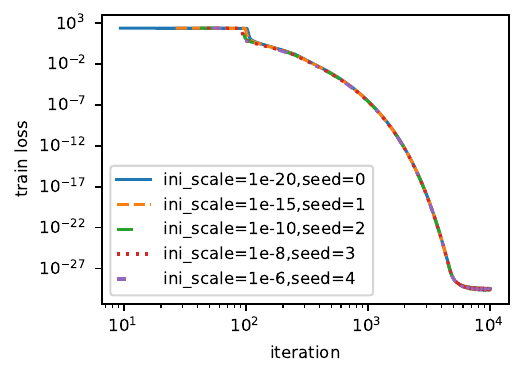}
\label{fig:c2_loss}}
\subfigure[$c=2$]{
\includegraphics[width=0.58\linewidth]{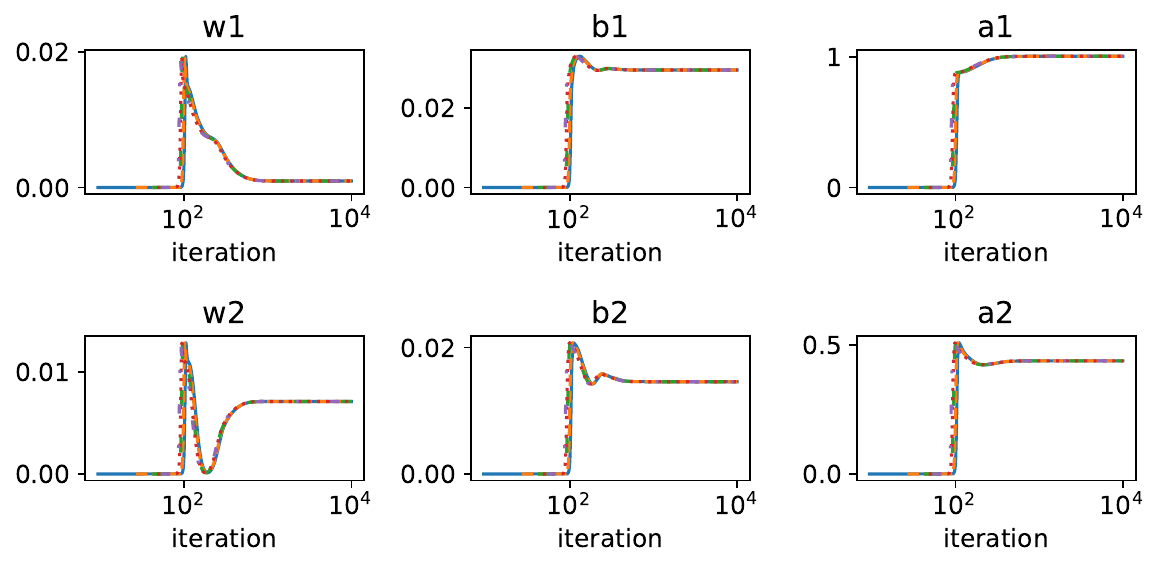}
\label{fig:c2_trajectory}}
\caption{Training a two-layer neural network with varying widths. The network model is $f_{\vtheta}(\vx)=a_1\tanh(\vw_1^{\top}\vx+b_1)+a_2\tanh(\vw_2^{\top}\vx+b_2)$, with $x \in \sR^3$. The target function is $f^*(\vx)=\tanh(\mathbf{1}^T\vx+1)$. The training data $(\vx_i,y_i)_{i=1}^n$ is obtained by $y_i=f^*(\vx_i)$ and $\{\vx_i\}_{i=1}^n$ draw independently from uniform distribution on interval $[-2,2]^d$ with random seed $0$. In this experiments, $m=2$,$d=500$,$n=500$. The iterations is $10^4$, the learning rate is $0.001$. Different initialization use different seeds to generate random number for Gaussian distribution of parameters. But we keep $c:=C_1/C_2$ to be a constant for all initialization in a subfigure. The value of $c$ is in the caption of subfigure. In the right part, the parameters $w_1$ and $w_2$ is the first coordinate of $\vw_1$ and $\vw_2$.}
\label{high dimensional loss and parameters}
\end{figure}

\begin{figure}[htbp]
\centering
\subfigure[$w^1$]{
\includegraphics[width=0.315\linewidth]{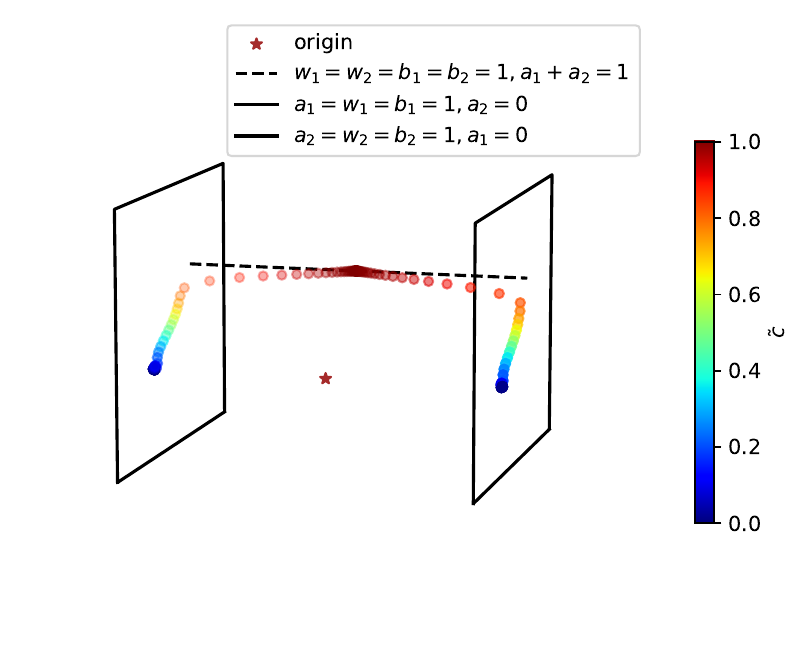}
\label{fig:w1}}
\subfigure[$w^2$]{
\includegraphics[width=0.315\linewidth]{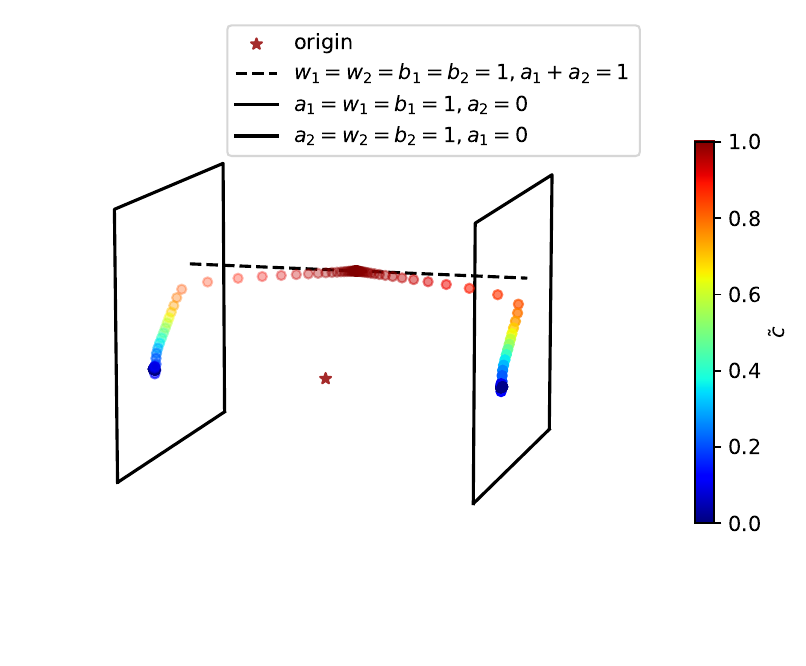}
\label{fig:w2}}
\subfigure[$w^3$]{
\includegraphics[width=0.315\linewidth]{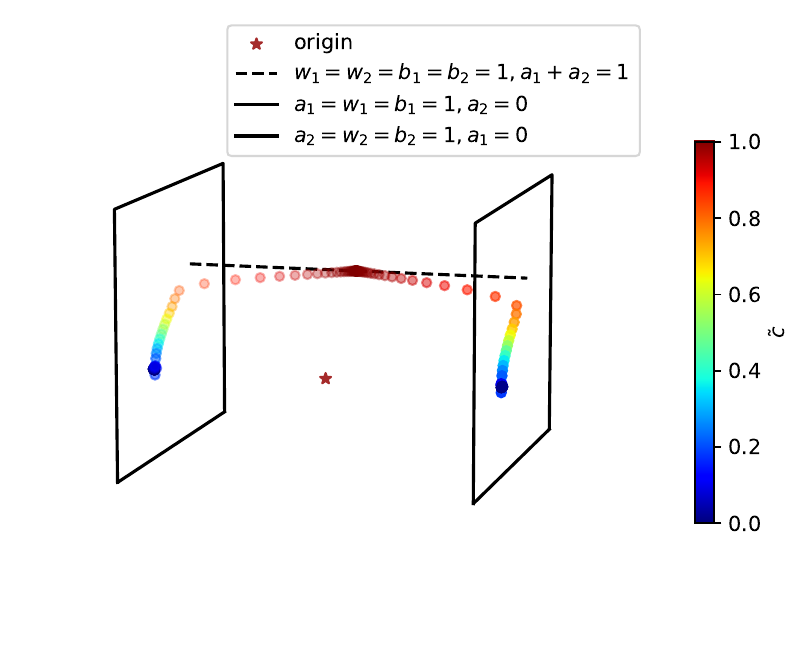}
\label{fig:w3}}
\caption{Training a two-layer neural network. The network model is $f_{\vtheta}(\vx)=a_1\tanh(\vw_1^{\top}\vx+b_1)+a_2\tanh(\vw_2^{\top}\vx+b_2)$, with $x \in \sR^3$. The target function is $f^*(\vx)=\tanh(\mathbf{1}^T\vx+1)$. The training data $(\vx_i,y_i)_{i=1}^n$ is obtained by $y_i=f^*(\vx_i)$ and $\{\vx_i\}_{i=1}^n$ draw independently from uniform distribution on interval $[-2,2]^d$. In this experiments, $m=2$,$d=3$,$n=10$. The iterations is $10^6$, the learning rate is $0.5$. The initialization scale is $10^{-20}$. Caption $w^i$ represents the $i$-th dimension of $\vw$.}
\label{high dimensional Q^*}
\end{figure}

\begin{figure}[htbp]
\centering
\includegraphics[width=0.6\linewidth]{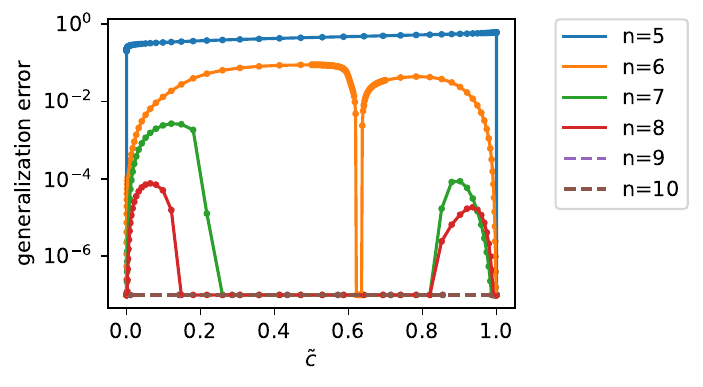}
\caption{Training a two-layer neural network. The network model is $f_{\vtheta}(\vx)=a_1\tanh(\vw_1^{\top}\vx+b_1)+a_2\tanh(\vw_2^{\top}\vx+b_2)$, with $x \in \sR^3$. The target function is $f^*(\vx)=\tanh(\mathbf{1}^T\vx+1)$. The training data $(\vx_i,y_i)_{i=1}^n$ is obtained by $y_i=f^*(\vx_i)$ and $\{\vx_i\}_{i=1}^n$ draw independently from uniform distribution on interval $[-2,2]^d$. In this experiments, $m=2$,$d=3$. The iterations is $10^6$ for $n=5,6,7,8$. The iterations is $2\times 10^7$ for $n=9,10$. The learning rate is $0.25$. The initialization scale is $10^{-20}$. The generalization error is obtained by assessing $1000$ points drawn independently from uniform distribution on interval $[-2,2]^d$. We identify generalization error lower than $10^{-7}$ to be $10^{-7}$, and regard it as successful recovery of the target function.}
\label{high dimensional, c and genloss}
\end{figure}

\begin{figure}[htbp]
\centering
\subfigure[$n=6$]{
\includegraphics[width=0.315\linewidth]{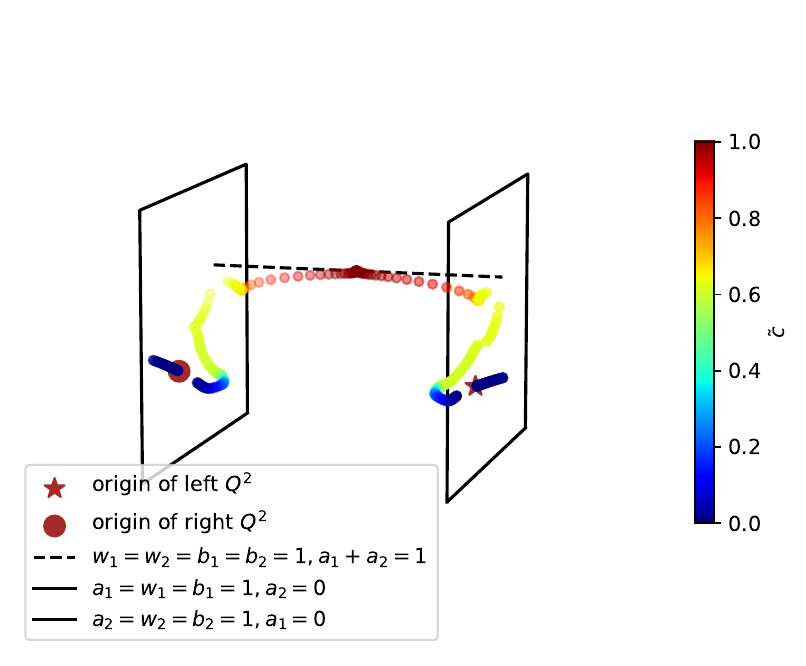}
\label{fig:n6}}
\subfigure[$n=7$]{
\includegraphics[width=0.315\linewidth]{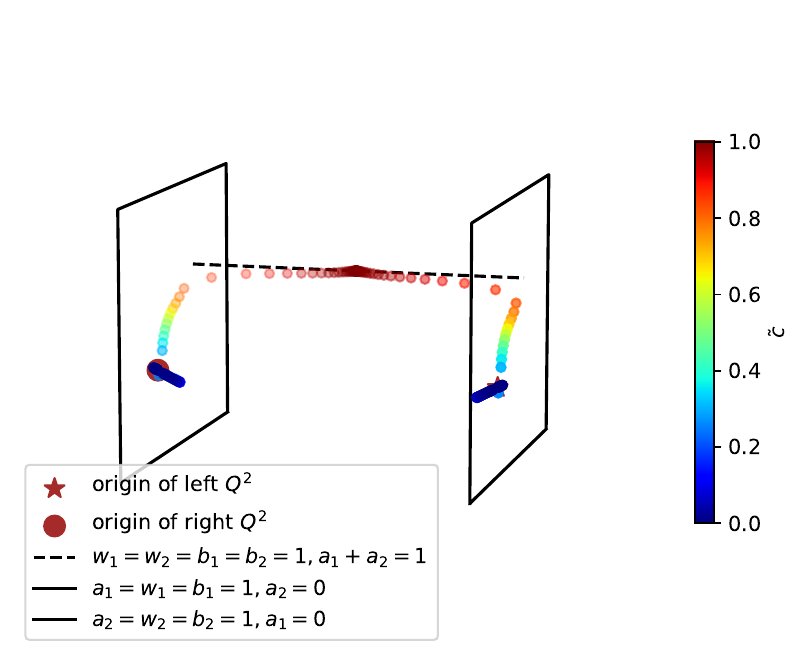}
\label{fig:n7}}
\subfigure[$n=8$]{
\includegraphics[width=0.315\linewidth]{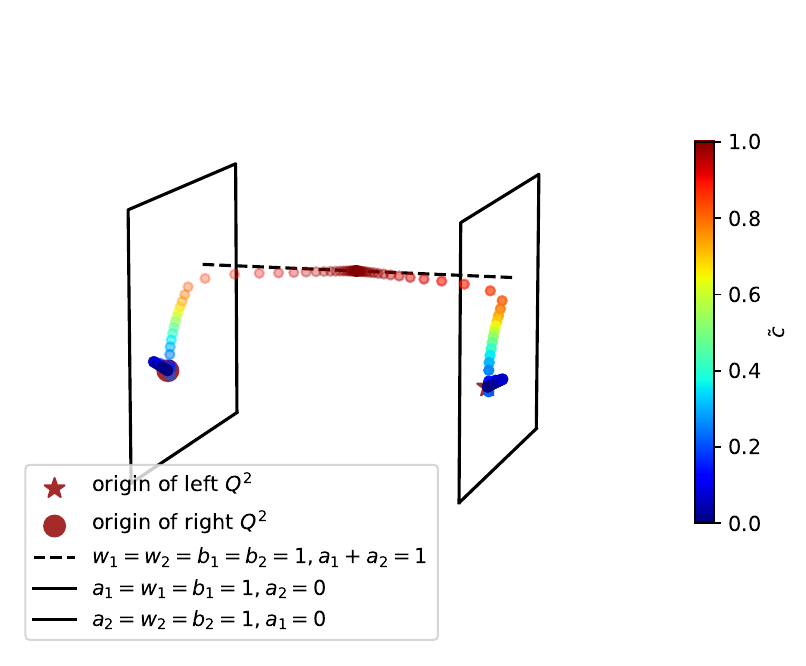}
\label{fig:n8}}
\caption{The experiment is same with \ref{high dimensional, c and genloss}. In this figure, the $w_1$ is first coordinate of $\vw_1$, $w_2$ is the first coordinate of $\vw_2$.}
\label{high dimension, converge of different n}
\end{figure}

\section{Experimental Details}
\label{sec:experimental setup}
In Figure~\ref{n=2} to ~\ref{n=6}, the learning rate was set to $0.5$. Points for calculating generalization error were $1000$ points evenly from the interval $[-2, 2]$. For Figures~\ref{n=5} and \ref{n=6}, due to nonlinear convergence of seeds, it is difficult to get a extremely low training loss. So we train the network until loss is $10^{-8}$. For For Figures~\ref{n=2} to ~\ref{n=6}, we train the network until training loss is $10^{-15}$. In Figure~\ref{scale2}, the generalization error is computed by $1000$ points $(x_i,y_i)_{i=1}^{1000}$ with $y_i=f^*(x_i)$ and $\{x_i\}_{i=1}^{1000}$ following i.i.d standard Gaussian distribution.  For $n=2,3$, the iterations is $10^6$. For $n=4,5$, the iterations is $4\times10^5$. The training was halted once the loss reached $10^{-15}$.

In Figure~\ref{effect of c}, training was performed using gradient descent with a learning rate of $0.01$. Suppose $\vtheta_1$ and $\vtheta_2$ to be the initialization of the first and the second neuron, respectively. We transform $\vtheta_2$ into $a\vtheta_2$, and choose appropriate value of $a$ to keep $c=0.5$ across all trials.  

In Figure~\ref{c and converging}, Training was conducted using gradient descent with a learning rate of $0.05$ and iterations of $10^6$. The dataset $(x_i,y_i)_{i=1}^n$ consists of 6 points, with $y_i=f^*(x_i)$ and $\{x_i\}_{i=1}^6$ equally spaced points on the interval $[-2,2]$. We use seeds $0$ to $400$ to generate initialization of parameters.


In Figure~\ref{c and genloss}, gradient descent was employed as the training algorithm with a learning rate of $0.5$. For $n=2,3,4$, in all experiments training loss reaches $10^{-15}$ and then  training stops. For $n=5,6$, the network is trained with iterations $10^7$. The train loss is shown in Figure~\ref{c and trainloss}.

\begin{figure}[htbp]
\centering
\subfigure[$n=5$]{
\includegraphics[width=0.4\linewidth]{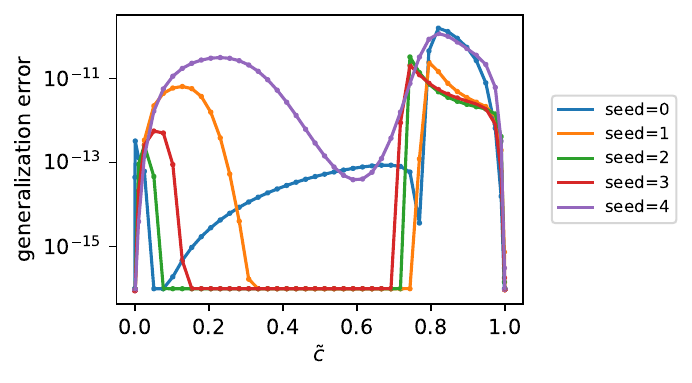}
\label{n=5,c and trainloss}}
\subfigure[$n=6$]{
\includegraphics[width=0.4\linewidth]{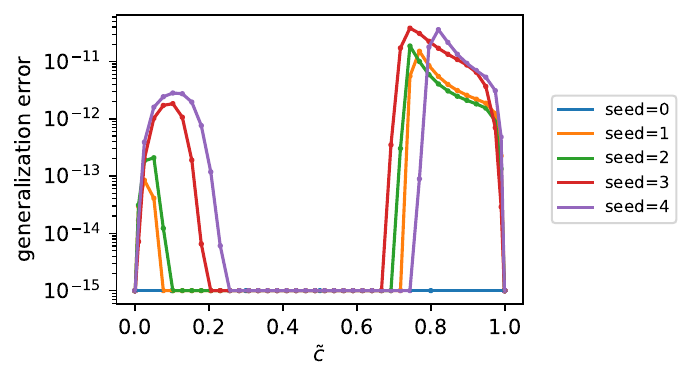}
\label{n=6,c and trainloss}}
\caption{For $n=5$, all networks are trained with learning rate $0.5$ and iterations $10^7$. The training stops when loss reaches $10^{-15}$.  For $n=6$, networks with seed$=1$ to seed$=4$ are trained with learning rate $0.5$ and iterations $2\times10^6$. The training stops when loss reaches $10^{-15}$. Networks with seed$=0$ are trained until loss reaches $10^{-15}$.
}
\label{c and trainloss}
\end{figure}

In Figure~\ref{multi neuron, effect of c}, gradient descent with a learning rate of $0.01$ is used for training. The iterations is $10^5$.
Suppose $\vtheta_1$ and $\vtheta_2$ to be the initialization of the first and the second neuron, respectively. We transform $\vtheta_2$ into $a\vtheta_2$, and choose appropriate value of $a$ to keep $c=0.5$ across all trials.

In Figure~\ref{multi neuron recovery}, gradient descent with a learning rate of $0.01$ is used, and the initial weights are drawn from a Gaussian distribution with a mean of $0$ and a standard deviation of $10^{-20}$. In all experiments, train loss reaches $10^{-15}$. For all width, we use $0$ as random seed to generate Gaussian distribution for the initialization of parameters.

The details of experiments of Figure~\ref{high dimensional loss and parameters} to Figure~\ref{high dimension, converge of different n} are in their captions.


\begin{thebibliography}{21}
\expandafter\ifx\csname natexlab\endcsname\relax\def\natexlab#1{#1}\fi
\providecommand{\url}[1]{\texttt{#1}}
\providecommand{\href}[2]{#2}
\providecommand{\path}[1]{#1}
\providecommand{\DOIprefix}{doi:}
\providecommand{\ArXivprefix}{arXiv:}
\providecommand{\URLprefix}{URL: }
\providecommand{\Pubmedprefix}{pmid:}
\providecommand{\doi}[1]{\href{http://dx.doi.org/#1}{\path{#1}}}
\providecommand{\Pubmed}[1]{\href{pmid:#1}{\path{#1}}}
\providecommand{\bibinfo}[2]{#2}
\ifx\xfnm\relax \def\xfnm[#1]{\unskip,\space#1}\fi
\bibitem[{Zhang et~al.(2023{\natexlab{a}})Zhang, Zhang, Zhang, Bai, Luo, and Xu}]{zhang2023optimistic}
\bibinfo{author}{Y.~Zhang}, \bibinfo{author}{Z.~Zhang}, \bibinfo{author}{L.~Zhang}, \bibinfo{author}{Z.~Bai}, \bibinfo{author}{T.~Luo}, \bibinfo{author}{Z.-Q.~J. Xu},
\newblock \bibinfo{title}{Optimistic estimate uncovers the potential of nonlinear models},
\newblock \bibinfo{journal}{arXiv preprint arXiv:2307.08921}  (\bibinfo{year}{2023}{\natexlab{a}}).
\bibitem[{Zhang et~al.(2023{\natexlab{b}})Zhang, Zhang, and Luo}]{zhang2023structure}
\bibinfo{author}{L.~Zhang}, \bibinfo{author}{Y.~Zhang}, \bibinfo{author}{T.~Luo}, \bibinfo{title}{Structure and gradient dynamics near global minima of two-layer neural networks}, \bibinfo{year}{2023}{\natexlab{b}}. \href{http://arxiv.org/abs/2309.00508}{{\tt arXiv:2309.00508}}.
\bibitem[{Vapnik(1998)}]{vapnik1998adaptive}
\bibinfo{author}{V.~N. Vapnik},
\newblock \bibinfo{title}{Adaptive and learning systems for signal processing communications, and control},
\newblock \bibinfo{journal}{Statistical learning theory}  (\bibinfo{year}{1998}).
\bibitem[{Bartlett and Mendelson(2002)}]{bartlett2002rademacher}
\bibinfo{author}{P.~L. Bartlett}, \bibinfo{author}{S.~Mendelson},
\newblock \bibinfo{title}{Rademacher and gaussian complexities: Risk bounds and structural results},
\newblock \bibinfo{journal}{Journal of Machine Learning Research} \bibinfo{volume}{3} (\bibinfo{year}{2002}) \bibinfo{pages}{463--482}.
\bibitem[{Breiman(2018)}]{breiman2018reflections}
\bibinfo{author}{L.~Breiman},
\newblock \bibinfo{title}{Reflections after refereeing papers for nips},
\newblock in: \bibinfo{booktitle}{The Mathematics of Generalization}, \bibinfo{publisher}{CRC Press}, \bibinfo{year}{2018}, pp. \bibinfo{pages}{11--15}.
\bibitem[{Zhang et~al.(2021)Zhang, Bengio, Hardt, Recht, and Vinyals}]{zhang2021understanding}
\bibinfo{author}{C.~Zhang}, \bibinfo{author}{S.~Bengio}, \bibinfo{author}{M.~Hardt}, \bibinfo{author}{B.~Recht}, \bibinfo{author}{O.~Vinyals},
\newblock \bibinfo{title}{Understanding deep learning (still) requires rethinking generalization},
\newblock \bibinfo{journal}{Communications of the ACM} \bibinfo{volume}{64} (\bibinfo{year}{2021}) \bibinfo{pages}{107--115}.
\bibitem[{Jiang et~al.(2019)Jiang, Neyshabur, Mobahi, Krishnan, and Bengio}]{jiang2019fantastic}
\bibinfo{author}{Y.~Jiang}, \bibinfo{author}{B.~Neyshabur}, \bibinfo{author}{H.~Mobahi}, \bibinfo{author}{D.~Krishnan}, \bibinfo{author}{S.~Bengio},
\newblock \bibinfo{title}{Fantastic generalization measures and where to find them},
\newblock \bibinfo{journal}{arXiv preprint arXiv:1912.02178}  (\bibinfo{year}{2019}).
\bibitem[{Zhang et~al.(2022)Zhang, Zhang, Zhang, Bai, Luo, and Xu}]{zhang2022linear}
\bibinfo{author}{Y.~Zhang}, \bibinfo{author}{Z.~Zhang}, \bibinfo{author}{L.~Zhang}, \bibinfo{author}{Z.~Bai}, \bibinfo{author}{T.~Luo}, \bibinfo{author}{Z.-Q.~J. Xu},
\newblock \bibinfo{title}{Linear stability hypothesis and rank stratification for nonlinear models},
\newblock \bibinfo{journal}{arXiv preprint arXiv:2211.11623}  (\bibinfo{year}{2022}).
\bibitem[{Yehudai and Ohad(2020)}]{yehudai2020learning}
\bibinfo{author}{G.~Yehudai}, \bibinfo{author}{S.~Ohad},
\newblock \bibinfo{title}{Learning a single neuron with gradient methods},
\newblock in: \bibinfo{booktitle}{Conference on Learning Theory}, \bibinfo{organization}{PMLR}, \bibinfo{year}{2020}, pp. \bibinfo{pages}{3756--3786}.
\bibitem[{Vardi et~al.(2021)Vardi, Yehudai, and Shamir}]{vardi2021learning}
\bibinfo{author}{G.~Vardi}, \bibinfo{author}{G.~Yehudai}, \bibinfo{author}{O.~Shamir},
\newblock \bibinfo{title}{Learning a single neuron with bias using gradient descent},
\newblock \bibinfo{journal}{Advances in Neural Information Processing Systems} \bibinfo{volume}{34} (\bibinfo{year}{2021}) \bibinfo{pages}{28690--28700}.
\bibitem[{Xu and Du(2023)}]{xu2023over}
\bibinfo{author}{W.~Xu}, \bibinfo{author}{S.~Du},
\newblock \bibinfo{title}{Over-parameterization exponentially slows down gradient descent for learning a single neuron},
\newblock in: \bibinfo{booktitle}{The Thirty Sixth Annual Conference on Learning Theory}, \bibinfo{organization}{PMLR}, \bibinfo{year}{2023}, pp. \bibinfo{pages}{1155--1198}.
\bibitem[{Vempala and Wilmes(2018)}]{Vempala2018PolynomialCO}
\bibinfo{author}{S.~Vempala}, \bibinfo{author}{J.~Wilmes},
\newblock \bibinfo{title}{Polynomial convergence of gradient descent for training one-hidden-layer neural networks},
\newblock \bibinfo{journal}{arXiv preprint arXiv:1805.02677}  (\bibinfo{year}{2018}).
\bibitem[{Wu(2022)}]{pmlr-v151-wu22c}
\bibinfo{author}{L.~Wu},
\newblock \bibinfo{title}{Learning a single neuron for non-monotonic activation functions},
\newblock in: \bibinfo{booktitle}{International Conference on Artificial Intelligence and Statistics}, \bibinfo{organization}{PMLR}, \bibinfo{year}{2022}, pp. \bibinfo{pages}{4178--4197}.
\bibitem[{Frei et~al.(2020)Frei, Cao, and Gu}]{frei2020agnostic}
\bibinfo{author}{S.~Frei}, \bibinfo{author}{Y.~Cao}, \bibinfo{author}{Q.~Gu},
\newblock \bibinfo{title}{Agnostic learning of a single neuron with gradient descent},
\newblock \bibinfo{journal}{Advances in Neural Information Processing Systems} \bibinfo{volume}{33} (\bibinfo{year}{2020}) \bibinfo{pages}{5417--5428}.
\bibitem[{Chistikov et~al.(2024)Chistikov, Englert, and Lazic}]{chistikov2024learning}
\bibinfo{author}{D.~Chistikov}, \bibinfo{author}{M.~Englert}, \bibinfo{author}{R.~Lazic},
\newblock \bibinfo{title}{Learning a neuron by a shallow relu network: Dynamics and implicit bias for correlated inputs},
\newblock \bibinfo{journal}{Advances in Neural Information Processing Systems} \bibinfo{volume}{36} (\bibinfo{year}{2024}).
\bibitem[{Oymak and Soltanolkotabi(2019)}]{oymak2019overparameterized}
\bibinfo{author}{S.~Oymak}, \bibinfo{author}{M.~Soltanolkotabi},
\newblock \bibinfo{title}{Overparameterized nonlinear learning: Gradient descent takes the shortest path?},
\newblock in: \bibinfo{booktitle}{International Conference on Machine Learning}, \bibinfo{organization}{PMLR}, \bibinfo{year}{2019}, pp. \bibinfo{pages}{4951--4960}.
\bibitem[{Safran et~al.(2022)Safran, Vardi, and Lee}]{safran2022effective}
\bibinfo{author}{I.~Safran}, \bibinfo{author}{G.~Vardi}, \bibinfo{author}{J.~D. Lee},
\newblock \bibinfo{title}{On the effective number of linear regions in shallow univariate relu networks: Convergence guarantees and implicit bias},
\newblock \bibinfo{journal}{Advances in Neural Information Processing Systems} \bibinfo{volume}{35} (\bibinfo{year}{2022}) \bibinfo{pages}{32667--32679}.
\bibitem[{Zhou et~al.(2022)Zhou, Qixuan, Luo, Zhang, and Xu}]{zhou2022towards}
\bibinfo{author}{H.~Zhou}, \bibinfo{author}{Z.~Qixuan}, \bibinfo{author}{T.~Luo}, \bibinfo{author}{Y.~Zhang}, \bibinfo{author}{Z.-Q. Xu},
\newblock \bibinfo{title}{Towards understanding the condensation of neural networks at initial training},
\newblock \bibinfo{journal}{Advances in Neural Information Processing Systems} \bibinfo{volume}{35} (\bibinfo{year}{2022}) \bibinfo{pages}{2184--2196}.
\bibitem[{Panageas et~al.(2019)Panageas, Piliouras, and Wang}]{panageas2019first}
\bibinfo{author}{I.~Panageas}, \bibinfo{author}{G.~Piliouras}, \bibinfo{author}{X.~Wang},
\newblock \bibinfo{title}{First-order methods almost always avoid saddle points: The case of vanishing step-sizes},
\newblock \bibinfo{journal}{Advances in Neural Information Processing Systems} \bibinfo{volume}{32} (\bibinfo{year}{2019}).
\bibitem[{Li et~al.(2020)Li, Luo, and Lyu}]{li2020towards}
\bibinfo{author}{Z.~Li}, \bibinfo{author}{Y.~Luo}, \bibinfo{author}{K.~Lyu},
\newblock \bibinfo{title}{Towards resolving the implicit bias of gradient descent for matrix factorization: Greedy low-rank learning},
\newblock \bibinfo{journal}{arXiv preprint arXiv:2012.09839}  (\bibinfo{year}{2020}).
\bibitem[{Lojasiewicz(1965)}]{lojasiewicz1965ensembles}
\bibinfo{author}{S.~Lojasiewicz},
\newblock \bibinfo{title}{Ensembles semi-analytiques},
\newblock \bibinfo{journal}{Lectures Notes IHES (Bures-sur-Yvette)}  (\bibinfo{year}{1965}).

\end{thebibliography}
\end{document}